\documentclass[nohyperref]{article}

\usepackage{microtype}
\usepackage{graphicx}
\usepackage{subfigure}
\usepackage{booktabs} 

\usepackage{caption}
\usepackage{hyperref}

\usepackage[accepted]{icml2022}

\usepackage{amsmath}
\usepackage{amssymb}
\usepackage{mathtools}
\usepackage{amsthm}
\usepackage{bm}

\usepackage[capitalize,noabbrev]{cleveref}

\theoremstyle{plain}
\newtheorem{theorem}{Theorem}[section]

\newtheorem{lemma}[theorem]{Lemma}

\theoremstyle{definition}
\newtheorem{definition}[theorem]{Definition}
\newtheorem{assumption}[theorem]{Assumption}
\theoremstyle{remark}

\usepackage[textsize=tiny]{todonotes}

\def\eg{\emph{e.g.}}
\def\ie{\emph{i.e.}}

\def\Fcal{{\mathcal F}}

\def\Ncal{{\mathcal N}}

\def\th{{\theta_{1:N}}}
\def\dw{{d_w}}
\def\dth{{d_{\theta}}}

\def\R{{\mathbb R}}

\DeclareMathOperator{\Tr}{Tr}

\DeclareMathOperator{\diag}{diag}

\DeclareMathOperator{\E}{\mathbb{E}}

\icmltitlerunning{Personalization Improves Privacy--Accuracy Tradeoffs in Federated Learning}

\begin{document}

\twocolumn[
\icmltitle{Personalization Improves Privacy--Accuracy Tradeoffs\\in Federated Learning}

\icmlsetsymbol{equal}{*}

\begin{icmlauthorlist}
\icmlauthor{Alberto Bietti}{nyu}
\icmlauthor{Chen-Yu Wei}{usc}
\icmlauthor{Miroslav Dudík}{msr}
\icmlauthor{John Langford}{msr}
\icmlauthor{Zhiwei Steven Wu}{cmu}

\end{icmlauthorlist}

\icmlaffiliation{nyu}{Center for Data Science, New York University}
\icmlaffiliation{usc}{University of Southern California}
\icmlaffiliation{msr}{Microsoft Research, New York}
\icmlaffiliation{cmu}{Carnegie Mellon University}

\icmlcorrespondingauthor{Alberto Bietti}{alberto.bietti@nyu.edu}

\icmlkeywords{Machine Learning, ICML}

\vskip 0.3in
]



\printAffiliationsAndNotice{}  

\begin{abstract}

Large-scale machine learning systems often involve data distributed across a collection of users.
Federated learning algorithms leverage this structure by communicating model updates to a central server, rather than entire datasets.
In this paper, we study stochastic optimization algorithms for a personalized federated learning setting involving local and global models subject to user-level (joint) differential privacy.
While learning a private global model induces a cost of privacy, local learning is perfectly private.
We provide generalization guarantees showing that coordinating local learning with private centralized learning yields a generically useful and improved tradeoff between accuracy and privacy. 
We illustrate our theoretical results with experiments on synthetic and real-world datasets.

\end{abstract}

\section{Introduction}
\label{sec:introduction}

Many modern applications of machine learning involve data from a large set of users. In such settings, both privacy considerations and  bandwidth limits may require keeping each user's data on their device,
instead communicating with a centralized server via shared model updates, a scenario commonly known as federated learning~\citep{kairouz2019advances}.
When the data distribution varies across users, it is often beneficial to consider personalized models where parts of the model (\eg, user-specific embeddings, or additive offsets) are local to each user, leading to updates that may be performed locally.
Such personalized federated learning has been deployed at scale in a wide range of applications, including keyboard next-word prediction~\citep{wang2019federated}, automated speech recognition, and news personalization~\citep{paulik2021federated}.\looseness=-1

In this paper, we focus on personalized federated learning algorithms based on stochastic optimization, which are the most widely used in this context~\citep{mcmahan2017communication,wang2021field}.
Concretely, we consider stochastic optimization problems of the following form:
\begin{equation}
\label{eq:opt}
\min_{w, \th} \left\{f(w, \th) := \frac{1}{N} \sum_{i=1}^N \E_{\xi \sim P_i}  [f_i(w, \theta_i, \xi)] \right\},
\end{equation}
where~$N$ is the number of users, $w$ a global parameter, $\th = (\theta_1, \ldots, \theta_N)$ a set of local parameters, and $P_i$ denotes the sample distribution for user~$i$.
Similar objectives have been considered in other works~\citep{agarwal2020federated,hanzely2021personalized}.
The personalization problem raises the question of how much learning should happen at the local vs global level: if the local models are expressive enough, local learning alone should suffice to learn good models with enough samples per user, yet if the data has some shared structure across users, a global model may help us learn more efficiently.

Although federated learning algorithms do not store user data in the central server, private user information may still be leaked in the resulting models, unless appropriate mechanisms are applied to guarantee privacy.
In this work, we consider the notion of \emph{user-level (joint) differential privacy} (DP) \cite{dwork2006calibrating, panprivacy, KPRU}, which ensures an adversary cannot reliably detect the presence or absence of all the data associated with a single user based on the output information. Such a notion has been used successfully in federated learning problems, with practical optimization algorithms that lead to useful privacy guarantees~\citep{mcmahan2018learning,hu2021private}.

Unfortunately, assuring such privacy may come at a cost to accuracy (see, for example, the lower bounds of~\citealp{bassily2014private}). The key question we address is:

{\bf Can we leverage personalization to improve privacy--accuracy tradeoffs in federated learning?}

The insight motivating our work is that in personalized federated learning, it is only the \emph{global} portion of the optimization that needs to experience this drop in accuracy.   A  user's local model can compensate for the privacy-guaranteeing accuracy limitations of the global model.
We formalize this by considering algorithms with a \emph{personalization parameter}~$\alpha$ that may vary the level of personalization from local learning ($\alpha = 0$) to global learning ($\alpha = \infty$).
We then show generalization bounds on the objective~\ref{eq:opt} of the form
\begin{equation}
\label{eq:generic_bound}
f(w_n, \theta_{1:N,n}) - f^* \leq C_\text{stat}(\alpha, n, N) + C_\text{priv}(\alpha, \epsilon, N),
\end{equation}
where~$f^* = \min_{w, \th} f(w, \th)$ is the optimal risk, $n$~is the number of observed samples per user and $\epsilon$ is the DP privacy parameter.
Here, we expect the second term to vanish when~$\alpha \to 0$, as local learning does not suffer from privacy. Crucially, this privacy cost does not depend on the number of samples per user~$n$, while the first term, which captures statistical efficiency, generally decreases with~$n$ for any~$\alpha$. This emphasizes how adjusting the level of personalization through~$\alpha$ can help improve generalization by adjusting the trade-off between these two terms.

Concretely, we provide precise guarantees of this form for simple federated private stochastic gradient algorithms, where the number of iterations corresponds to the number of samples per user~$n$, and the personalization parameter~$\alpha$ is given by the relative step-size between global and local updates.
In particular, we show that~$\alpha$ affects the complexity of learning by changing the geometry of the optimization: in problems that benefit from global models, small~$\alpha$ makes learning more difficult but reduces the cost of privacy.
We complement our theoretical results with experiments on synthetic and real-world federated learning datasets, which illustrate how varying the step-size ratio leads to improved trade-offs between accuracy and privacy.

\section{Related Work}

\paragraph{Model personalizaton.} There are a variety of approaches for personalization in federated learning. In \emph{local fine-tuning}, a global model is learnt by federated learning and then used as a warm start for on-device learning from the cache of local data~\citep{wang2019federated,paulik2021federated}. This approach can be augmented with federated learning of hyperparameters~\citep{wang2019federated,jiang2019improving,khodak2019adaptive} to obtain federated learning variants of meta-learning approaches like MAML~\citep{finn2017modelagnostic} and Reptile~\citep{nichol2018firstorder}.

Another approach is to view personalization as a \emph{multi-task learning problem}, and learn task-specific models with a regularization that forces parameters for similar tasks to be close, \eg,
in a Euclidean norm~\citep{vanhaesebrock2017decentralized,smith2017federated,arivazhagan2019federated,mansour2020approaches,dinh2020personalized,shen2020federated,huang2021personalized,marfoq2021federated,singhal2021federated}. Task similarity is typically expressed as a weighted undirected graph, a matrix, or a clustering. A special case of the multi-task approach is to learn a global model in addition to local models, which are regularized to be close to the global model~\citep{mansour2020approaches,deng2020adaptive,hanzely2020federated,hanzely2021personalized,marfoq2021federated}.
This is closest to the approach here, which also separates global and local parameters. However, the approach here is more general, because we allow a broader range of modeling relationships between the global and local parameters, similar to federated residual learning~\citep{agarwal2020federated}.
The statistical aspects of such personalization models were studied by~\citet{mansour2020approaches,agarwal2020federated}, who in particular provide generalization guarantees for additive personalization models similar to ours.
However, compared to the present paper, these works do not provide privacy guarantees nor study the effect of varying the level of personalization.\looseness=-1

\paragraph{Privacy.}
The results here advance a recent line of work that provides formal user-level privacy guarantees for model personalizaton in federated learning. Similar to the prior work of \citet{jain2021differentially} and \citet{hu2021private}, we adopt the privacy formulation of \emph{joint differential privacy}~\citep[JDP, ][]{KPRU}, a variation of DP that is more suitable for the problem of model personalization than standard DP. \citet{jain2021differentially} provides private algorithms that first learn a shared linear representation for all users, and allow each user to learn their local linear regression model over the lower-dimensional representation. This leads to a factorized model of personalization, which is different than ours and not handled by our theoretical assumptions due to non-convexity. \citet{hu2021private} provides a private personalizaton algorithm through the mean-regularized multi-task learning objective without establishing its statistical rates. In contrast, we consider more general personalization schemes and provide statistical guarantees.
\looseness=-1

The results here are also related to other work in DP with a similar motivation to model personalization. \citet{li2020differentially} studies meta-learning under DP. Their framework does not cover model personalization with a separate model for each user. \citet{noble2021differentially} studies federated learning with DP with heterogeneous data across nodes, but they do not support personalization. \citet{bellet2018personalized} studies fully decentralized DP algorithms for collaborative learning over a network instead of federated learning.

Finally, the notion of user-level privacy has also been adopted in prior works \citep{mcmahan2018learning,levy2021learning}, but these do not consider model personalization.

\section{Preliminaries}
\label{sec:preliminaries}

In this section, we introduce the problem of personalized federated optimization, as well as the notion of user-level privacy that we consider.

\subsection{Problem setting}
\label{sub:prelim_setting}

We consider stochastic optimization problems of the form
\begin{equation*}
\min_{w, \th} \left\{f(w, \th) := \frac{1}{N} \sum_{i=1}^N f_i(w, \theta_i) \right\},
\end{equation*}
where~$N$ is the number of users, $w \in \R^{\dw}$ a global parameter, $\th = (\theta_1, \ldots, \theta_N) \in (\R^{\dth})^N$ a set of local parameters, and~$f_i(w, \theta_i) := \E_{\xi\sim P_i}[f_i(w, \theta_i, \xi)]$ is the expected risk of user~$i$, with random samples~$\xi$ drawn from an unknown user-specific distribution~$P_i$.

While our algorithms may be run on arbitrary differentiable models, our analysis focuses on the convex setting, where~$f_i(w, \theta_i, \xi_i)$ is jointly convex in~$(w, \theta_i)$ for all~$\xi_i$.

\paragraph{Additive model.}
An important special case is the additive model for supervised learning, where $\dw=\dth=d$, and
\begin{equation}
\label{eq:additive_model}
f_i(w, \theta_i, (x, y)) = \ell(y, (w + \theta_i)^\top x),
\end{equation}
where~$\ell$ is a loss function and~$\xi = (x, y)$ is a training sample.\looseness=-1

As a running example, consider a movie recommendation app, which seeks to predict how each user $i$ will rate any given movie.
In this case, $x$ corresponds to the features describing a movie (e.g., its genre, popularity, length, actors, etc.), and $y$ is the user's rating of that movie.
Following the federated protocol, the data about user activity stays on the device and only model parameters can be communicated to the server. In this case, only the information about the global parameter $w$ is communicated.\looseness=-1

The additive model makes the optimization problem~\eqref{eq:opt} underdetermined, with many possible equivalent solutions obtained by adding any vector $v\in\R^d$ to~$w$ and then subtracting $v$ from each~$\theta_i$, effectively ``shifting'' the predictive ability between global and local parameters. This is what allows the optimization algorithm to achieve different tradeoffs between statistical generalization (accuracy) and sharing of information across users (privacy).

To develop the intuition about this tradeoff, first consider the homogeneous scenario, where the optimal parameters~$\theta_i^* \in \arg\min_\theta \E_{(x, y) \sim P_i} [\ell(y, \theta^\top x)]$ for each user are equal ($\theta_1^* = \cdots = \theta_N^*$). There are two extreme approaches: (i) \textbf{local learning}, where only~$\theta_i$ are trained individually for each user, leading to poor sample complexity but perfect privacy (ii) \textbf{global learning}, where only~$w$ is trained, and we benefit from using samples from all users, but need communication with a centralized server and lose privacy.

In the more realistic heterogeneous scenario, when the~$\theta_i^*$ are different but have some shared components, \eg, only some coordinates differ across users, \textbf{joint learning} of~$w$ and~$\th$ can achieve greater accuracy (at the same number of samples) than both local and global learning, by leveraging more samples to estimate the shared components, while benefiting from user-specific data to personalize. Note that in this case, it is possible to further improve privacy, by allowing some shared components to be fitted entirely locally.
Quantifying this improvement is the focus of our work.

In the movie recommendation example, if only global learning is performed, the prediction $w^\top x$ can use overall popularity statistics,
but the resulting prediction is the same for all users, without any personalization. On the other hand, if only local learning is performed, although the system can fully personalize and provide full privacy, the quality of recommendation is limited by the number of movies the user watched before. With joint learning, the system can capture both global trends through $w$ and adapt to each user's preference through $\theta_i$.

\subsection{User-level (joint) differential privacy}
\label{sub:prelim_privacy}

We aim to provide \emph{user-level} privacy \cite{panprivacy}, which ensures that an adverary cannot detect the presence or absence of \emph{all of the data associated with a single user} when given the output of the algorithm. To achieve this goal, we design algorithms using the \emph{billboard model} \cite{HsuHRRW14} of \emph{differential privacy} (DP) \cite{dwork2006calibrating}. Let us first revisit the definition of DP, which informally requires that changing any single user's data cannot change the algorithm's output by much.

\begin{definition}[User-level DP;~\citealp{dwork2006calibrating, panprivacy}]%
\label{def:DP}%
A randomized mechanism $M$ is  $(\epsilon,\delta)$-differential privacy (DP) if for all pairs of data sets $D,D'$ that differ by a single user's data and all events $E$ in the output range,
\begin{align*}
    \Pr[ M(D) \in E ] \leq e^\epsilon \Pr[ M(D') \in E ] + \delta.
\end{align*}
\end{definition}

\paragraph{Billboard model.} In the billboard model, a server computes aggregate information subject to the constraint of DP and shares the information as public messages with all $N$ users. Then, based on the public messages and their own private data, each user computes their own personalized model. The billboard model is particularly compatible with algorithms in the federated setting, where the DP messages are typically noisy summary statistics about the users' local models \cite{jain2021differentially, hu2021private}.
\citet{HsuHRRW14} show that algorithms under the billboard model provide an extremely strong privacy guarantee,  known as \emph{joint differential privacy}~\citep[JDP,][]{KPRU}.

\paragraph{Joint differential privacy (JDP).} Let $D_i$ denote the collection of samples associated with each user $i$. Two datasets $D$ and $D'$ are called $i$-neighbors if they only differ by user $i$'s private data. For any mechanism $M$, we denote $M_{-i}(D)$ as the output information to all other users except user $i$.

\begin{definition}[Joint-Differential Privacy, JDP]
\label{def:jointzCDP}
An algorithm $M$ is $(\epsilon, \delta)$-jointly differentially private, written as $(\epsilon, \delta)$-JDP,  if for all $i$, all $i$-neighboring datasets $D, D'$, and all events $E$ in the output space to all other users except $i$, 
\begin{align*}
        \Pr[ M_{-i}(D)\in E] \leq e^\epsilon \Pr[ M_{-i}(D') \in E] + \delta.
\end{align*}
\end{definition}

In the setting of model personalization, JDP implies that even if all of users except $i$ collude, potentially pooling their private data and local models together, user $i$'s private data are still protected, so long as $i$ does not reveal their own~model.

\section{Main Algorithm and Analysis}
\label{sec:algorithm}

\begin{algorithm}[tb]
   \caption{Personalized-Private-SGD (PPSGD)}
   \label{alg:sgd}
\begin{algorithmic}[1]
   \STATE {\bfseries Input:} $\eta$: step-size, $\alpha$: global/local ratio,\\
          \hphantom{\bfseries Input:} $\sigma_\zeta$: privacy noise level, $C$: clipping parameter.
   \STATE Initialize $w_0 = \theta_0 = 0$.
   \FOR{$t=1$ {\bfseries to} $n$}
   \FOR{all clients $i$ in parallel}
   \STATE Sample data $\xi_{i,t} \sim P_i$
   \STATE Compute $g_{\theta,i}^t = \nabla_{\theta} f_i(w_{t-1}, \theta_{i,t-1}, \xi_{i,t})$\\
   \hphantom{Compute} $g_{w,i}^t = \nabla_w f_i(w_{t-1}, \theta_{i,t-1}, \xi_{i,t})$
   \STATE Update $\theta_{i,t} = \theta_{i,t-1} - \frac{\eta}{N} g_{\theta,i}^t$
   \STATE Clip gradient: $\tilde g_{w,i}^t = g_{w,i}^t / \smash{\max(1, \frac{\|g_{w,i}^t\|}{C})}$
   \STATE Send $\tilde g_{w,i}^t$ to the server
   \ENDFOR
   \STATE Sample $\zeta_t \sim \Ncal(0, \sigma_\zeta^2 I_{\dw})$
   \STATE Update $w_t = w_{t-1} - \alpha \eta( \frac{1}{N} \sum_{i=1}^N \tilde g_{w,i}^t + \zeta_t)$
   \ENDFOR
\end{algorithmic}
\end{algorithm}

Our main algorithm, shown in Algorithm~\ref{alg:sgd}, is a personalized version of distributed SGD~\citep{dekel2012optimal}, with a \emph{personalization parameter}~$\alpha$ that controls the relative step-size between global and local updates.

At each round~$t$, each user~$i$ samples a fresh datapoint~$\xi_{i,t}$ (this could also be a mini-batch), updates its local model~$\theta_i$, and sends the gradient with respect to~$w$ to the central server, which aggregates gradients of all users before updating the global model~$w$.
The total number of rounds~$n$ thus corresponds to the number of samples per user.
In order to guarantee user-level privacy, we clip each user's gradients~$g_{w,i}^t$ before aggregation, and add Gaussian noise, following common practice~\citep{abadi2016deep,mcmahan2018learning,hu2021private, ChenWH20}.
The choice of~$\alpha$ affects the degree by which the algorithm favors local learning over global learning, with $\alpha=0$ forcing local-only~updates.

\paragraph{Privacy analysis.}
We first establish the formal user-level JDP  guarantee of  Algorithm~\ref{alg:sgd}.

\begin{theorem}[Privacy]
\label{prop:privacy}
Suppose we set the noise parameter
\begin{equation}
\label{eq:sigma_zeta}
\sigma_\zeta \geq c\frac{C\sqrt{n \log(1/\delta)}}{N \epsilon},
\end{equation}
for an absolute constant~$c$, then Algorithm~\ref{alg:sgd} satisfies $(\epsilon, \delta)$-JDP in the billboard model.
\end{theorem}

We defer the full proof to Appendix~\ref{sub:privacy_proof}.
At a high level, the proof first shows that the aggregate information released by the server (the sequence of global model updates) satisfies $(\epsilon, \delta)$-DP. Since the sequence of global models are sufficient statistics for each user to identify their personalized model $\theta_i$, the JDP guarantee follows from the billboard lemma (see Lemma \ref{lemma:billboard}). 

\paragraph{Generalization analysis.}
Our generalization analysis relies on the following assumptions. We begin with an assumption about minimizers of~$f$, which relies on the following norm for a vector~$z=(w, \th)$ that captures the geometry induced by the personalization parameter~$\alpha$:
\begin{equation}
\label{eq:alpha_norm}
\|z\|^2_\alpha := \frac{1}{\alpha}\|w\|^2 + \|\th\|^2.
\end{equation}
\begin{assumption}[Minimizers]
\label{ass:minimizer}
$f$ admits a minimizer~$z^*$ with finite norm~$\|z^*\|_\alpha$.
\end{assumption}
In the case of local learning ($\alpha = 0$), this implies~$z^*$ must have no global component.
For joint learning, different minimizers might exist, and our bounds scale with the minimal norm~$\|z^*\|_\alpha$ among all such minimizers.

\begin{assumption}[Convexity and smoothness]
\label{ass:conv_smooth}
For all~$i$ and $P_i$-almost every~$\xi$, the function $(w, \theta_i) \mapsto f_i(w, \theta_i, \xi)$ is jointly convex and~$L$-smooth (its gradients are~$L$-Lipschitz).
\end{assumption}

Note that this implies that~$f(z)$ is jointly convex in~$z=(w, \th)$.
If we consider the example~$f_i(w, \theta_i, (x, y)) = \frac{1}{2} (y - (w + \theta_i)^\top x)^2$, and assume~$\|x\| \leq R$ almost surely, it is easy to verify that the assumption holds with~$L = 2R^2$.

We will also make the following boundedness assumption on gradients, which is commonly made in the context of private stochastic optimization~\citep[\eg,][]{bassily2014private,feldman2020private}, and simplifies our analysis by avoiding the need to study the effect of clipping on optimization.
\begin{assumption}[Bounded gradients]
\label{ass:bounded_grad}
For all~$i$, $w$, $\theta_i$, and $P_i$-almost every~$\xi$, we have~$\|\nabla_w f_i(w, \theta_i, \xi)\| \leq G$. 
\end{assumption}
This assumptions avoid the need to clip gradients when~$C \geq G$, and our analysis hereafter assumes~$C=G$.
We note that~$G$ may be large in some cases, growing with the norm of optimal parameters, thus smaller values of~$C$ may often be beneficial in practice for better privacy guarantees.

\paragraph{Gradient variances.}
We consider the following variance quantities, which we assume to be finite, and which are obtained by considering gradients at a given minimizer~$z^*$:
\begin{align*}
    \sigma_{w,i}^2 &= \E_{\xi} \| \nabla_w f_i(w^*, \theta_i^*, \xi) - \nabla_w f_i(w^*, \theta_i^*) \|^2 \\
    \sigma_{\theta,i}^2 &= \E_{\xi} \|\nabla_\theta f_i(w^*, \theta_i^*, \xi)\|^2 \\
    \bar \sigma_w^2 &= \frac{1}{N} \sum_i \sigma_{w,i}^2, \qquad \bar \sigma_\theta^2 = \frac{1}{N} \sum_i \sigma_{\theta,i}^2.
\end{align*}
As an example, note that in a simple additive model~\eqref{eq:additive_model} with squared loss and additive label noise of variance~$\tau^2$, we have~$\sigma_{w,i}^2 = \sigma_{\theta,i}^2 =\tau^2 \Tr(\E_{P_i}[x x^\top])$, recovering standard statistical quantities. If the algorithm relies on mini-batches of size~$m$ instead of single datapoints, we may then replace~$\sigma_{w,i}^2$ and~$\sigma_{\theta,i}^2$ by~$\sigma_{w,i}^2/m$ and~$\sigma_{\theta,i}^2/m$, respectively, by averaging gradients over the mini-batches.

We now provide our main result on the convergence rate of Algorithm~\ref{alg:sgd}, which also yields a generalization bound on the excess risk. The proof is in Appendix~\ref{sub:convergence_proof}.

\begin{theorem}[Generalization]
\label{thm:convergence}
Under Assumptions~\ref{ass:minimizer},~\ref{ass:conv_smooth}, and~\ref{ass:bounded_grad}, let~$z^* = (w^*, \th^*)$ be any minimizer of~$f$,~$L_\alpha := L\max(\alpha, \frac{1}{N})$, and
\begin{equation}
\label{eq:sigma_tot}
    \sigma_{tot,\alpha}^2 := \frac{\alpha \bar \sigma_w^2 + \bar \sigma_\theta^2}{N} + \alpha d_w \sigma_\zeta^2.
\end{equation}
With~$\eta = \min \{\frac{1}{4 L_\alpha}, \frac{\|z^*\|_\alpha}{\sqrt{n} \sigma_{tot,\alpha}}\}$ and~$C=G$, Algorithm~\ref{alg:sgd} satisfies
\begin{align}
    \E[&f(\bar z_n) - f(z^*)] \leq \frac{4 L_\alpha \|z^*\|_\alpha^2}{n} + 3 \frac{\sigma_{tot,\alpha}\|z^*\|_\alpha}{\sqrt{n}},
\end{align}
with~$\bar z_n = \frac{1}{n} \sum_{t=0}^{n-1} z_t$.
In particular, with~$\sigma_\zeta$ as in~\eqref{eq:sigma_zeta}, hiding absolute constants, we have the following generalization bound:
\begin{align}
    &\E[f(\bar z_n) - f(z^*)] \lesssim \\
    &\frac{ L_\alpha \|z^*\|_\alpha^2}{n} +  \|z^*\|_\alpha \sqrt{ \frac{ \alpha \bar \sigma_w^2 + \bar \sigma_\theta^2}{N n}} + \|z^*\|_\alpha \sqrt{\frac{ \alpha \dw G^2 \log(\frac{1}{\delta})}{N^2 \epsilon^2}} \label{eq:excess_risk}
\end{align}
\end{theorem}

The generalization bound takes the form~\eqref{eq:generic_bound}, with a cost of privacy that does not depend on the number of samples per user~$n$.
As is common in the analysis of SGD, our bound displays a bias term that decays as~$1/n$, and a variance term decaying as~$1/\sqrt{n}$, controlled by the gradient variance~$\sigma_{tot,\alpha}^2$.
Ignoring the privacy noise, given that we use~$N$ samples at each round, this variance scales as~$1/N$, leading to an overall asymptotic rate of~$1/\sqrt{Nn}$, where~$Nn$ is the total sample size.
Using minibatches of size~$m$ would further improve this to~$1/\sqrt{Nnm}$, where the total number of samples is now~$Nnm$.

\paragraph{Local vs global learning.}
Note that when multiple minimizers~$z^*$ exist, one may choose the one with minimal norm~$\|z^*\|_\alpha$.
If we consider the additive model~\eqref{eq:additive_model}, we may always consider a minimizer of the form~$z^* = (0, \theta^*_{1:N})$. Then~$\|z^*\|_\alpha$ is bounded even for~$\alpha=0$ (\textbf{local learning}), which yields an excess risk of~order
\[
 \frac{L \sum_i \|\theta_i^*\|^2}{Nn} + \sqrt{\frac{\bar \sigma_\theta^2 \sum_i \|\theta_i^*\|^2}{Nn}}.
\]
In particular, there is no cost for privacy.
If we further assume~$\theta_i^* = v^*$ for all~$i$, this becomes
\[
 \frac{L \|v^*\|^2}{n} + \|v^*\| \sqrt{\frac{\bar \sigma_\theta^2 }{n}}.
\]
We see slow convergence that does not improve with~$N$, which is expected given that no information is shared across users.
In this setting, we may instead consider a different minimizer~$z^* = (v^*, 0)$, for which taking the limit~$\alpha \to \infty$ (\textbf{global learning}) yields the excess risk of order
\[
\frac{L \|v^*\|^2}{n} +  \|v^*\| \sqrt{\frac{\bar \sigma_w^2}{Nn}} + \|v^*\| \sqrt{\frac{\dw G^2 \log(\frac{1}{\delta})}{N^2 \epsilon^2}}.
\]
The variance term now decays faster for large~$N$, but has an additional privacy term, which decreases quickly with~$N$ but does not improve with large~$n$ and may be quite large, particularly in high dimensions, consistent with lower bounds for differentially private optimization~\cite{bassily2014private}.\looseness=-1

\paragraph{Characterizing benefits of personalization.}
Depending on the number of samples per user $n$, it may be helpful to choose varying levels of personalization~$\alpha$ to adjust this tradeoff, from large~$\alpha$ for small~$n$ to small~$\alpha$ for large~$n$ when the privacy cost dominates the bound.
The next lemma quantifies this tradeoff for the additive model with homogeneous users (\ie, $\theta_1^* = \cdots = \theta_N^*$):
\begin{lemma}[Threshold on~$n$ for personalization benefits]\label{lem: effect alpha}
Assume~$(v, 0)$ is a minimizer of~$f$, with~$f$ an additive model of the form~\eqref{eq:additive_model}.
For any~$\alpha$, the minimizer with minimal $\alpha$-norm is given by~$(w, \th)$ with~$w = \frac{\alpha N}{\alpha N + 1}v$ and~$\theta_i = \frac{1}{\alpha N + 1} v$.
Assuming~$\bar \sigma_w = \bar \sigma_\theta = \sigma$, the variance term in the excess risk~\eqref{eq:excess_risk} with~$n$ samples per user is monotonic in~$\alpha$, taking the form:
\[
\|v\| \sqrt{\frac{N}{\alpha N + 1} \left(\frac{(\alpha + 1)\sigma^2}{Nn} + \frac{\alpha \dw G^2 \log(1/\delta)}{N^2 \epsilon^2}\right)}.
\]
This is non-decreasing with~$\alpha$ if and only if
\begin{equation}
\label{eq:critical_T}
n \gtrsim \frac{N(N - 1) \sigma^2 \epsilon^2}{\dw G^2 \log(1/\delta)}.
\end{equation}
\end{lemma}
Thus, if we ignore the bias term in~\eqref{eq:excess_risk}, this suggests that when~\eqref{eq:critical_T} holds, using local learning (smaller~$\alpha$) should help improve generalization, while if the reverse inequality holds, global learning should be preferred.
Eq.~\eqref{eq:critical_T} suggests that this threshold on the user sample size~$n$ scales quadratically with the number of users~$N$, linearly with the privacy level~$\epsilon$, and inversely with the dimension~$\dw$.
When the optimal user parameters are different, this transition would likely happen for smaller~$n$, as we expect local models to be useful even at small sample sizes regardless of privacy.

\paragraph{Improving the bias term.}
We remark that the bias term decreases as~$1/n$ in both local and global learning scenarios described above, and does not improve with the number of users. As in standard SGD~\citep[\eg,][]{dekel2012optimal}, this term decreases with the number of rounds, and may thus decrease more quickly with the number of samples if more communication rounds are performed for the same total number of samples. In Section~\ref{sec:extensions}, we show that sampling users at each round can help improve this term.

\paragraph{Comparison to Local SGD.}
A common approach for federated optimization is the local SGD (or federated averaging) algorithm~\citep{mcmahan2017communication}, which performs multiple local steps before communicating with the server.
This is in contrast to our method, which more closely resembles mini-batch SGD.
We note that despite its practical success, 
the known theoretical guarantees of local SGD typically do not improve on mini-batch SGD except for specific settings~\citep{woodworth2020local}.
Our study therefore focuses on understanding personalization in the SGD setting, but we note that extending our analysis to local SGD is an interesting direction for future work.

\section{Heterogeneous Sample Sizes, User Sampling}
\label{sec:extensions}

\begin{algorithm}[tb]
   \caption{PPSGD with client sampling}
   \label{alg:sgd_ext}
\begin{algorithmic}[1]
   \STATE {\bfseries Input:} $q$: client sampling probability,\\
   \hphantom{\bfseries Input:} $m_i$: minibatch sizes, $\eta$: step size,\\
   \hphantom{\bfseries Input:} $\alpha$: global/local ratio,
                               $\sigma_\zeta$: privacy noise level,\\
   \hphantom{\bfseries Input:} $C$: clipping parameter.
   \STATE Initialize $w_0 = \theta_0 = 0$.
   \FOR{$t=1$ {\bfseries to} $T$}
  \STATE Sample~$b_{i,t} \sim Ber(q)$
   \FOR{all clients $i$ with~$b_{i,t} = 1$ in parallel}
   \STATE Sample a minibatch $\{\xi_{i,t}^{(k)}\}_{k=1}^{m_i} \sim P_i^{\otimes m_i}$
   \STATE Compute $g_{\theta,i}^t = \sum_{k=1}^{m_i} \nabla_{\theta} f_i(w_{t-1}, \theta_{i,t-1}, \xi_{i,t}^{(k)})$\\
   \hphantom{Compute} $g_{w,i}^t = \sum_{k=1}^{m_i} \nabla_w f_i(w_{t-1}, \theta_{i,t-1}, \xi_{i,t}^{(k)})$
   \STATE Update
   $\theta_{i,t} = \theta_{i,t-1} - \frac{\eta}{qM} g_{\theta,i}^t$
   \STATE Clip gradient $\tilde g_{w,i}^t = g_{w,i}^t / \smash{\max(1, \frac{\|g_{w,i}^t\|}{C})}$
   \STATE Send $\tilde g_{w,i}^t$ to the server
   \ENDFOR
   \STATE Sample $\zeta_t \sim \Ncal(0, \sigma_\zeta^2 I_{\dw})$
   \STATE Update $w_t = w_{t-1} - \alpha \eta (\frac{1}{qM}\sum_{i:b_{i,t}=1} \tilde g_{w,i}^t + \zeta_t)$
   \ENDFOR
\end{algorithmic}
\end{algorithm}

In this section, we study a variant of Algorithm~\ref{alg:sgd} where we sample some of the users uniformly at each round, with possibly different minibatch sizes for each user.
This reflects the fact that users may have different amounts of data, and that not all clients may be available at each round.

Let~$q$ be the probability of sampling any given user at each round, $m_i \geq 1$ be the mini-batch size for user~$i$, and define~$M = \sum_i m_i$ and~$m_{\max} = \max_i m_i$.
Algorithm~\ref{alg:sgd_ext} then optimizes the objective
\begin{equation}
\label{eq:fq}
    f^m(w, \th) := \sum_{i=1}^N\frac{m_i}{M} f_i(w, \theta_i).
\end{equation}
Namely, the algorithm optimizes each user's performance
proportionally to their amount of samples.
We denote the number of rounds by~$T$ here, since it no longer corresponds to the total number of samples per user~$n$.
Note that the expected total number of samples observed for user~$i$ after~$T$ rounds is now given by~$n_i = T q m_i$.

The following result provides a privacy guarantee for Algorithm~\ref{alg:sgd_ext}.
Compared to Theorem~\ref{prop:privacy}, it also leverages a standard argument of privacy amplification by subsampling.
\begin{theorem}[Privacy with client sampling]
\label{prop:privacy_sampling}
There exist absolute constants~$c_1$, $c_2$ such that for any~$\epsilon < c_1 q^2 T$, if we take
\begin{equation}
\label{eq:sigma_zeta_sampling}
\sigma_\zeta \geq c_2 \frac{ C \sqrt{T \log(1/\delta)}}{M \epsilon},
\end{equation}
then Algorithm~\ref{alg:sgd_ext} satisfies $(\epsilon, \delta)$-JDP.
\end{theorem}
Note that compared to the uniform sampling case, the clipping parameter~$C$ may need to be larger since we are clipping the sum of up to~$m_{\max}$ gradients.
In our generalization analysis hereafter, we thus assume~$C=G m_{\max}$, so that clipping is not needed under Assumption~\ref{ass:bounded_grad}.
We now give an optimization and generalization guarantee.
\begin{theorem}[Generalization with client sampling]
\label{thm:convergence_sampling}
Under Assumptions~\ref{ass:minimizer},~\ref{ass:conv_smooth}, and~\ref{ass:bounded_grad}, let~$z^* = (w^*, \th^*)$ be any minimizer of~$f^m$,~$L_{m,\alpha} = L\max\bigl\{\alpha + \frac{\alpha m_{\max}}{M},\, \frac{m_{\max}}{M}\bigr\}$, and
\begin{equation}
\label{eq:sigma_q}
    \sigma_{m,\alpha}^2 := \frac{\alpha \bar \sigma_{w,m}^2 + \alpha \tilde \sigma_{w,m}^2 + \bar \sigma_{\theta,m}^2}{qM} + \alpha d_w \sigma_\zeta^2,
\end{equation}
where~$\bar \sigma_{w,m}^2 := \frac{1}{M} \sum_i m_i \sigma_{w,i}^2$ and~$\bar \sigma_{\theta,m}^2 := \frac{1}{M} \sum_i m_i \sigma_{\theta,i}^2$ and~$\tilde \sigma_{w,m}^2 := \frac{1}{M} \sum_i q(1-q) m_i^2 \|\nabla_w f_i(w^*, \theta_i^*)\|^2$.

With~$\eta = \min\bigl\{ \frac{1}{4 L_{m,\alpha}}, \frac{\|z^*\|_\alpha}{\sqrt{T} \sigma_{m,\alpha}}\bigr\}$ and~$C=G_m := G m_{\max}$, Algorithm~\ref{alg:sgd_ext} satisfies
\begin{align*}
\!\!\!\!\!\!
    \E[&f^m(\bar z_T) - f^m(z^*)] \leq \frac{4 L_{m,\alpha} \|z^*\|_\alpha^2}{T} + 3 \frac{\sigma_{m,\alpha}\|z^*\|_\alpha}{\sqrt{T}},
\!\!\!\!\!\!
\end{align*}
with~$\bar z_T = \frac{1}{T} \sum_{t=0}^{T-1} z_t$.
Taking~$\sigma_\zeta$ as in~\eqref{eq:sigma_zeta_sampling}, we have:
\begin{align}
    &\E[f^m(\bar z_T) - f^m(z^*)] \lesssim \frac{L_{m,\alpha} \|z^*\|_\alpha^2}{T} + \\
    &+  \|z^*\|_\alpha \sqrt{ \frac{ \alpha \bar \sigma_{w,m}^2 \!+\! \alpha \tilde \sigma_{w,m}^2 + \bar \sigma_{\theta,m}^2}{q MT} \!+\! \frac{ \alpha \dw G_m^2 \log(\frac{1}{\delta})}{M^2 \epsilon^2}} \nonumber
\end{align}

\end{theorem}

Note that the privacy cost now scales with~$G_m / M =G m_{\max} / M$ instead of~$G / N$ in Theorem~\ref{thm:convergence}. Note that we have~$G m_{\max} / M \geq G / N$, with equality if and only if all the~$m_i$ are equal.
This highlights the additional cost of privacy with heterogeneous sample sizes, particularly when some users have much more data than others.
Compared to Theorem~\ref{thm:convergence}, there is also an additional variance term $\tilde{\sigma}_{w,m}$ induced by the randomness of client sampling.
Note that~$\tilde{\sigma}_{w,m}$ vanishes when (i) $q = 1$ (\ie, no sampling), or (ii) $\nabla_w f_i(w^*, \theta_i^*)=0$ for all~$i$, which always holds for the additive model~\eqref{eq:additive_model}.

When all users are sampled ($q=1$) and~$m_i = 1$ for each user, we have~$M = N$ and $\tilde{\sigma}_{w,m}=0$, and we recover the same bound as in Theorem~\ref{thm:convergence}, where~$T = n$ corresponds to the number of samples per user.
In this case the bias term decreases slowly as~$1/n$.
In contrast, if~$m_i = 1$ and~$q = 1/N$, \ie, we sample one user per round on average, then~$qM = 1$, and~$T$ now corresponds to the total number of samples~$Nn$. The bias term now decreases as~$1/Nn$, while the variance decreases at the same~$1/\sqrt{Nn}$ rate.
Of course, this comes at the cost of more frequent communication rounds, since each round only uses the data of a single user.

\paragraph{Optimizing average user performance.}
In some cases, we may have clients with heterogeneous amounts of data, but still want to optimize the average performance~$f$ in~\eqref{eq:opt}, rather than the weighted average~$f^m$ in~\eqref{eq:fq}.
This can be achieved by replacing~$g_{\theta,i}^t$ and~$g_{w,i}^t$ by averages over the minibatch instead of sums in Algorithm~\ref{alg:sgd_ext}.
This leads to the following guarantee on the excess risk (see Appendix~\ref{sub:risk_equal_nonu}):
\begin{align}
    &\E[f(\bar z_T) - f(z^*)] \lesssim \\
    &\frac{ L_\alpha \|z^*\|_\alpha^2}{T} +  \|z^*\|_\alpha \sqrt{ \frac{ \alpha \bar \sigma_w^2 + \bar \sigma_\theta^2}{q N m_{\min} T} + \frac{ \alpha \dw G^2 \log(\frac{1}{\delta})}{N^2 \epsilon^2}}. \label{eq:excess_risk_nonu} 
\end{align}
Note that the privacy cost is similar to the case with homogeneous data in Section~\ref{sec:algorithm}, rather than the larger cost of Theorem~\ref{thm:convergence_sampling}, since we are now treating each user equally.
Nevertheless, the variance term displays an effective total sample size of~$n_{\text{tot}} = qN m_{\min} T$, which is smaller than in Theorem~\ref{thm:convergence_sampling}, where~$n_{\text{tot}} = q M T$ corresponds to the expected total number of samples, unless all the~$m_i$ are equal. This highlights that this improvement in privacy comes at a cost in statistical efficiency.\looseness=-1

\section{Experiments}
\label{sec:experiments}

\begin{figure*}[ht!]
\centering
\includegraphics[width=.32\textwidth]{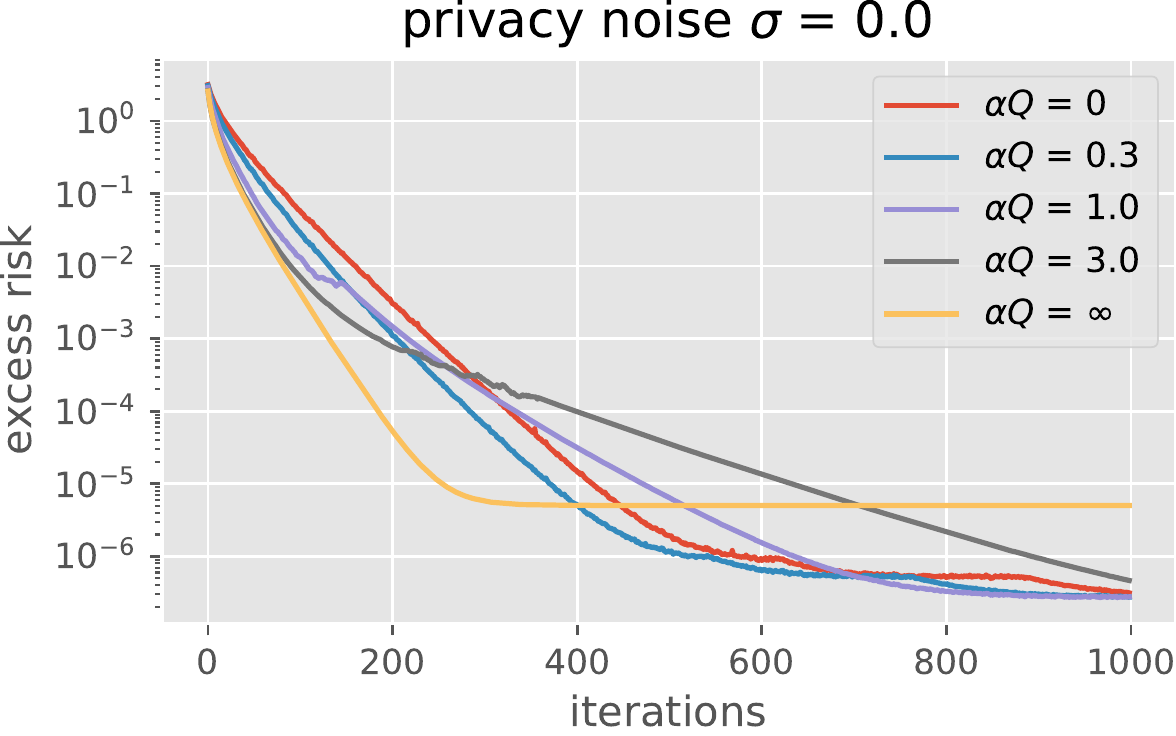}
\includegraphics[width=.33\textwidth]{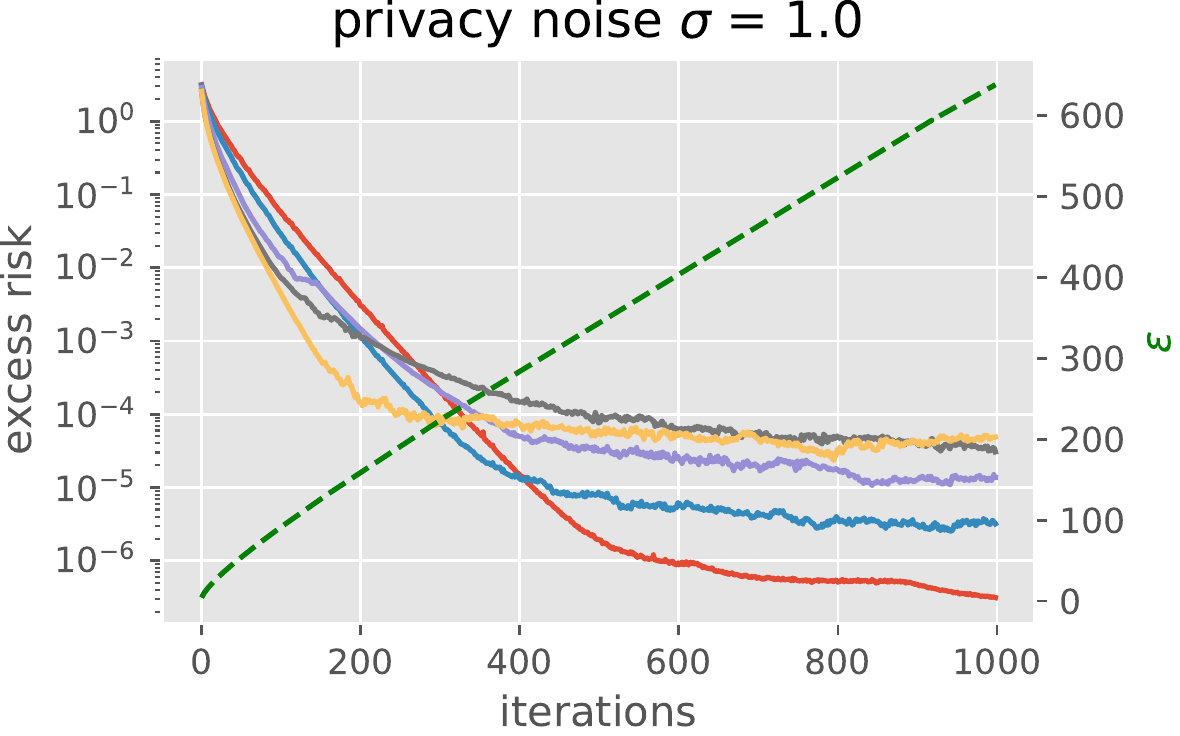}
\includegraphics[width=.33\textwidth]{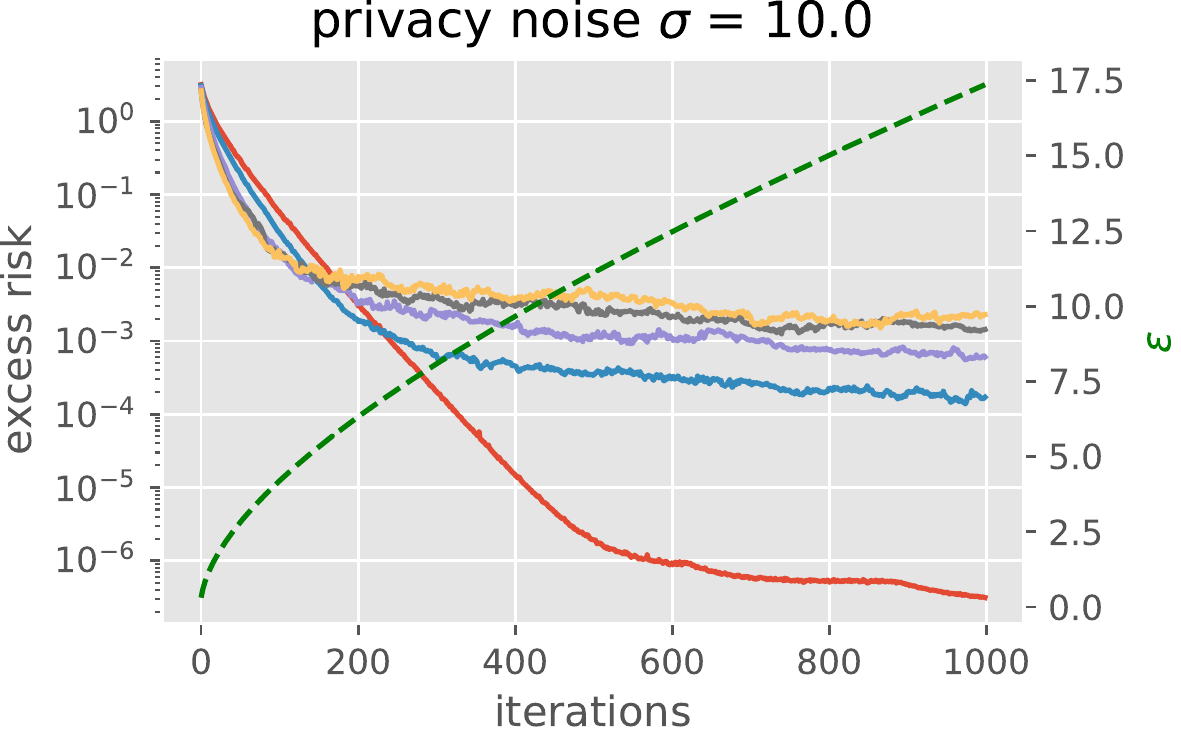}

(a) \textbf{Synthetic}: $N=1000$, $d=100$, $C=10$. One pass with~$Q=N$ and~$m=10$.

\vspace{0.5cm}

\includegraphics[width=.32\textwidth]{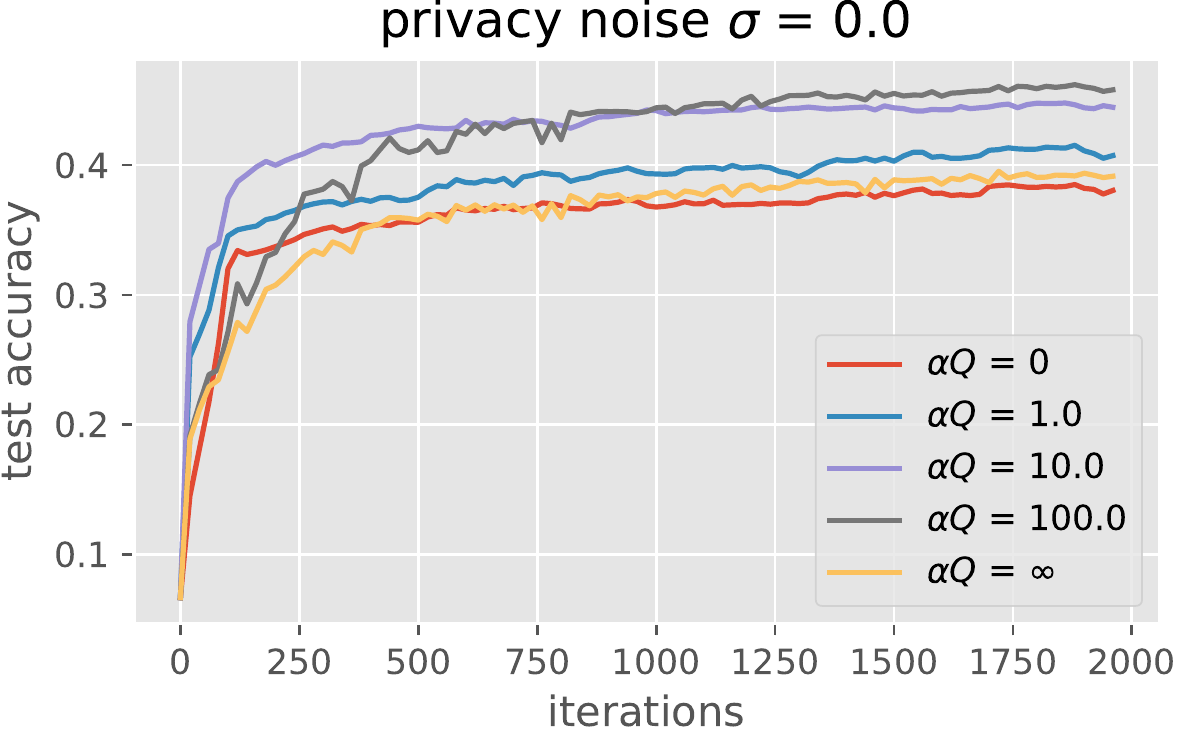}
\includegraphics[width=.33\textwidth]{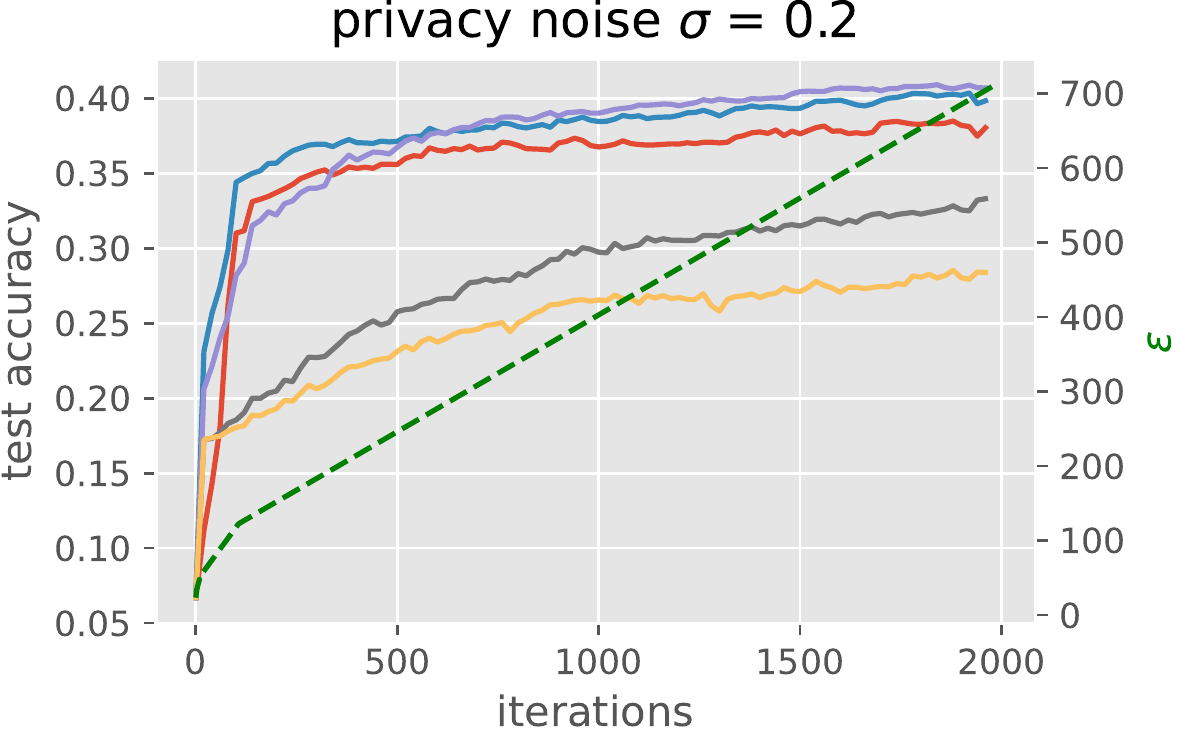}
\includegraphics[width=.33\textwidth]{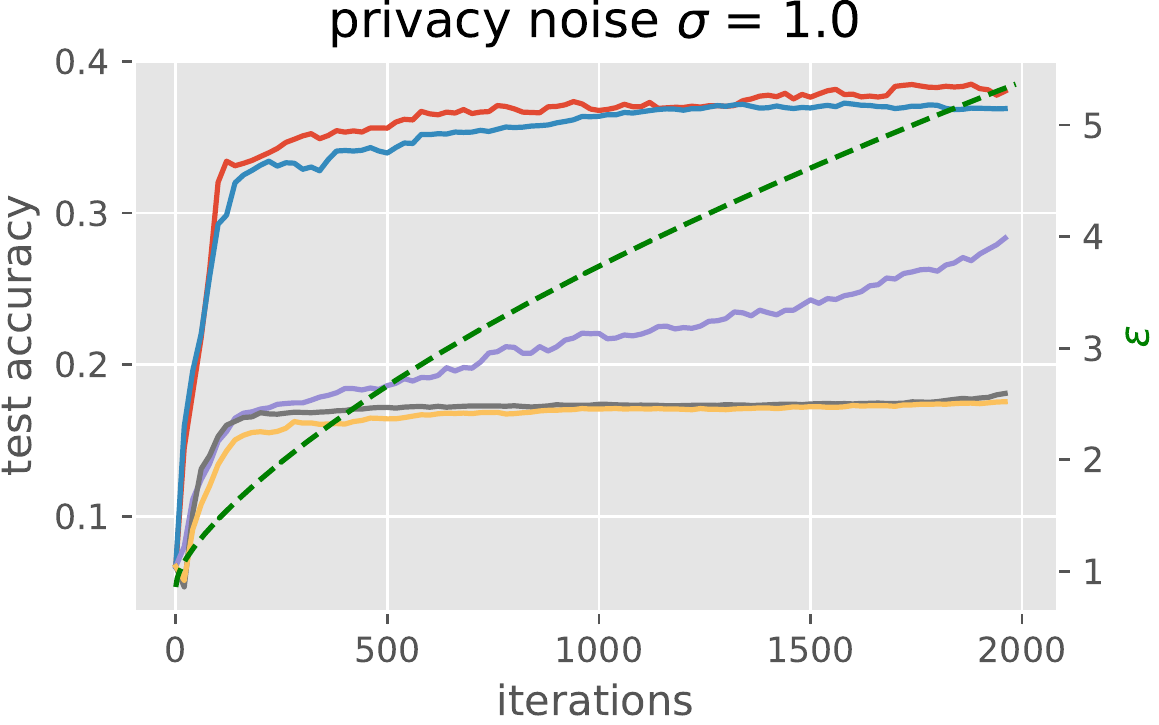}

(b) \textbf{Stackoverflow tag prediction}: $N=500$, $d=5000 \times 80$, $C=0.01$. One pass with~$Q=10$ and~$m=10$.

\vspace{0.5cm}

\includegraphics[width=.32\textwidth]{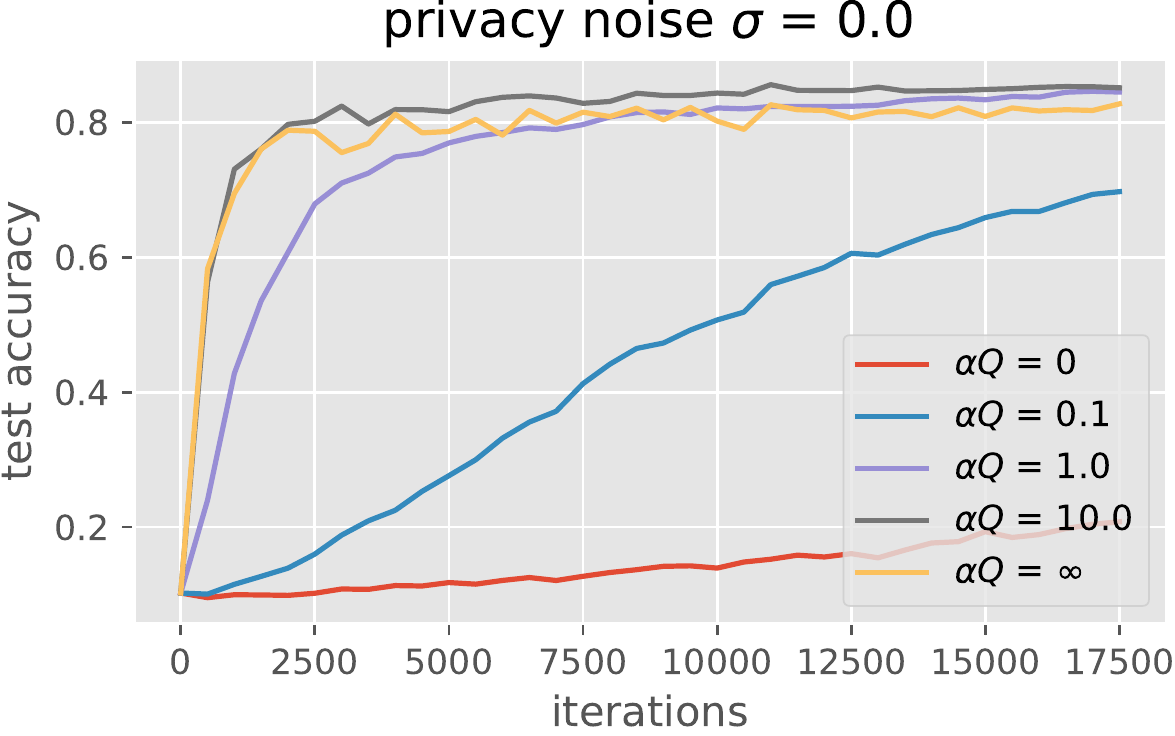}
\includegraphics[width=.33\textwidth]{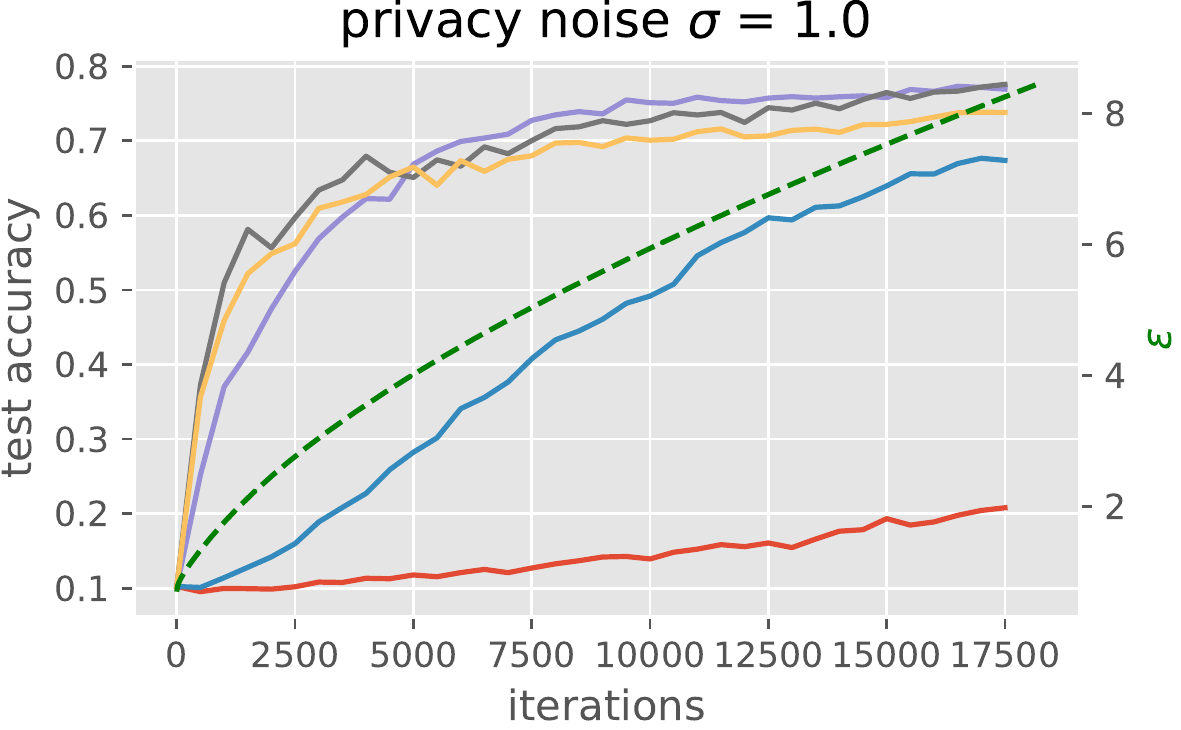}
\includegraphics[width=.33\textwidth]{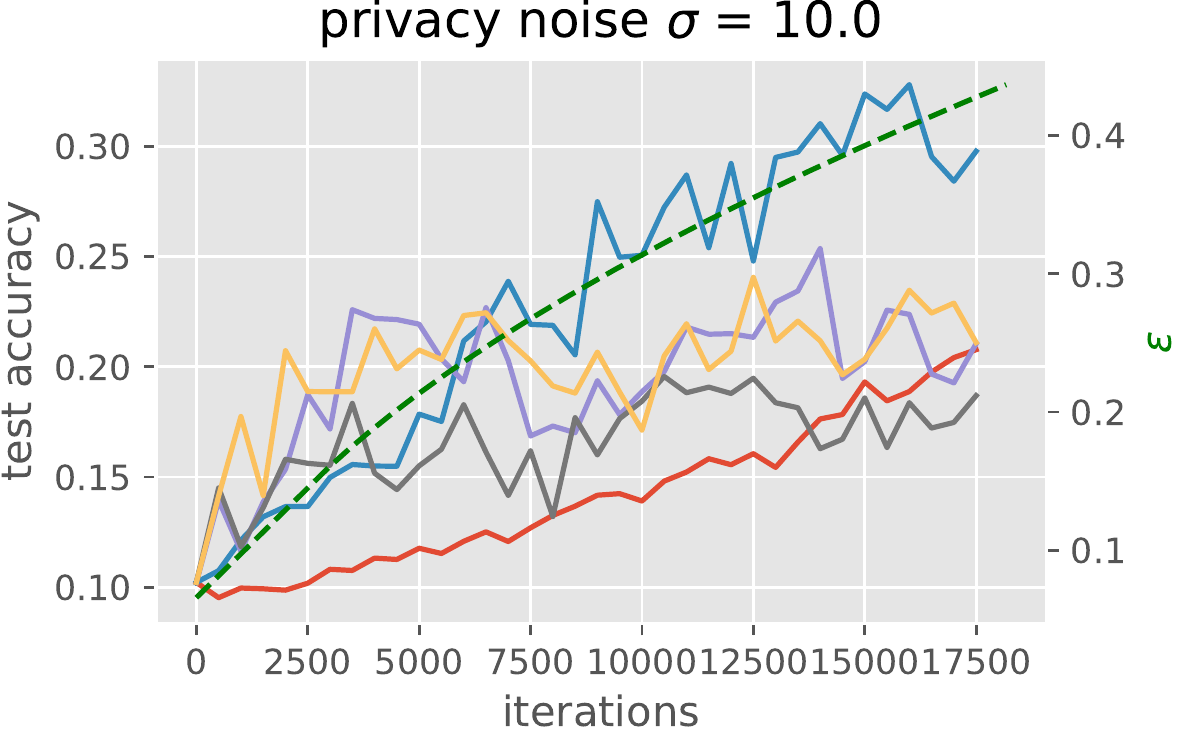}

(c) \textbf{EMNIST linear}: $N=1000$, $d=784 \times 10$, $C=10$.~~20 passes with~$Q=10$ and~$m=10$.

\vspace{0.5cm}

\includegraphics[width=.32\textwidth]{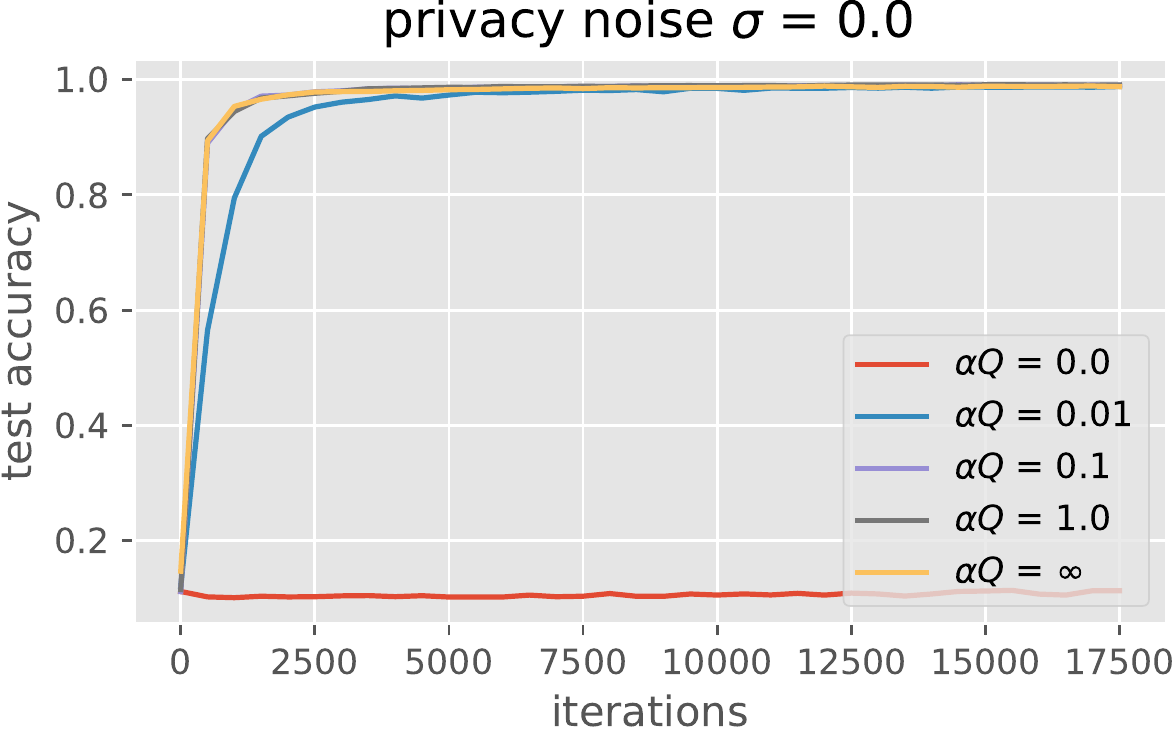}
\includegraphics[width=.33\textwidth]{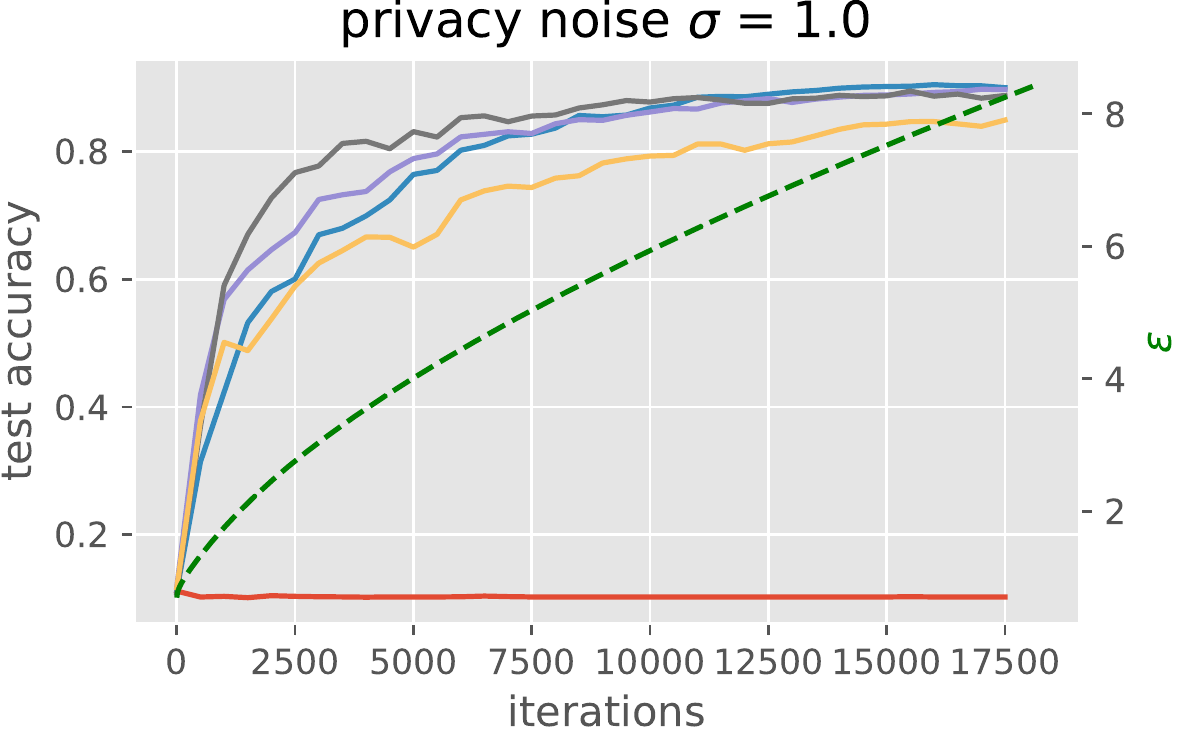}
\includegraphics[width=.33\textwidth]{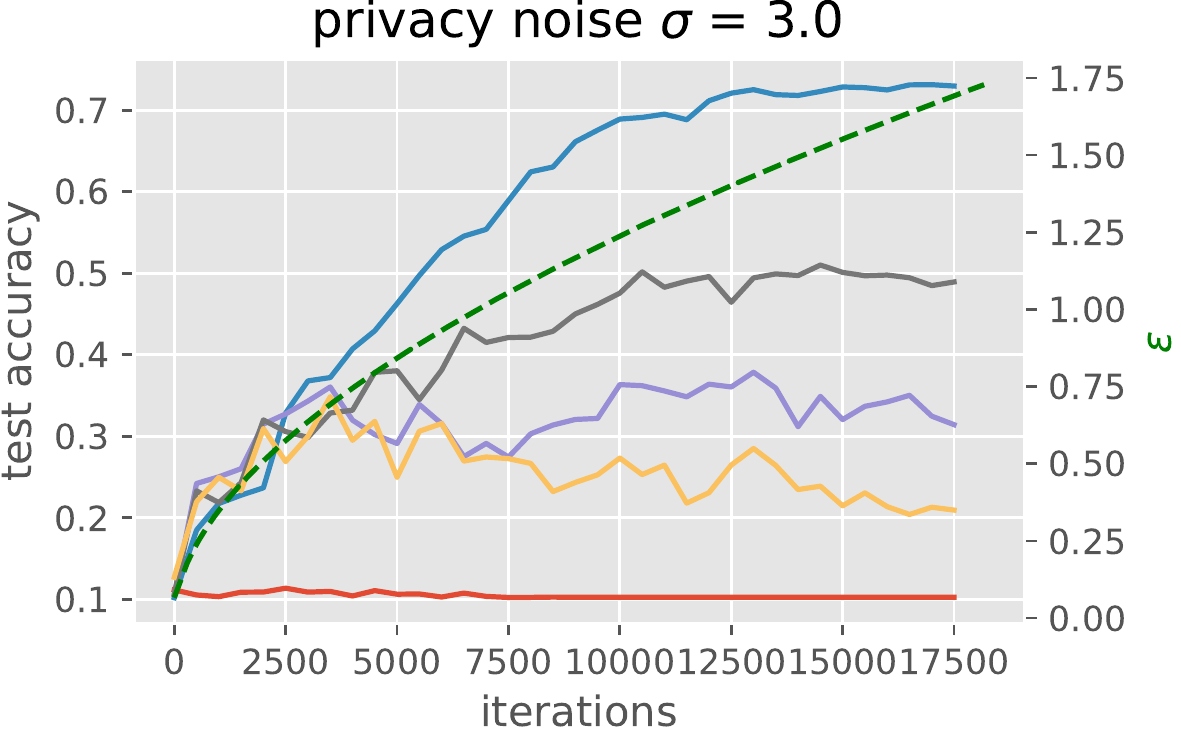}

(d) \textbf{EMNIST CNN}: $N=1000$, $d=784 \times 10$, $C=0.1$.~~20 passes with~$Q=10$ and~$m=10$.

\caption{Test performance of PP-SGD on the synthetic, Stackoverflow, and EMNIST datasets. Each plot shows curves for different levels of personalization~$\alpha$ for fixed privacy noise~$\sigma$ and clipping parameter~$C$. In the private setting, the green dashed line displays the privacy guarantee~$\epsilon$ at iteration~$T$ for all curves in each plot (except~$\alpha=0$ which is always fully private).
$Q$ is the number of sampled users and~$m$ the number of samples per user in each round.
The choice~$\alpha Q=1$ equalizes the variance of global and local updates due to samples.
}
\label{fig:opt_curves}
\end{figure*}

\begin{figure*}[htb]
    \centering
    \includegraphics[width=.30\textwidth]{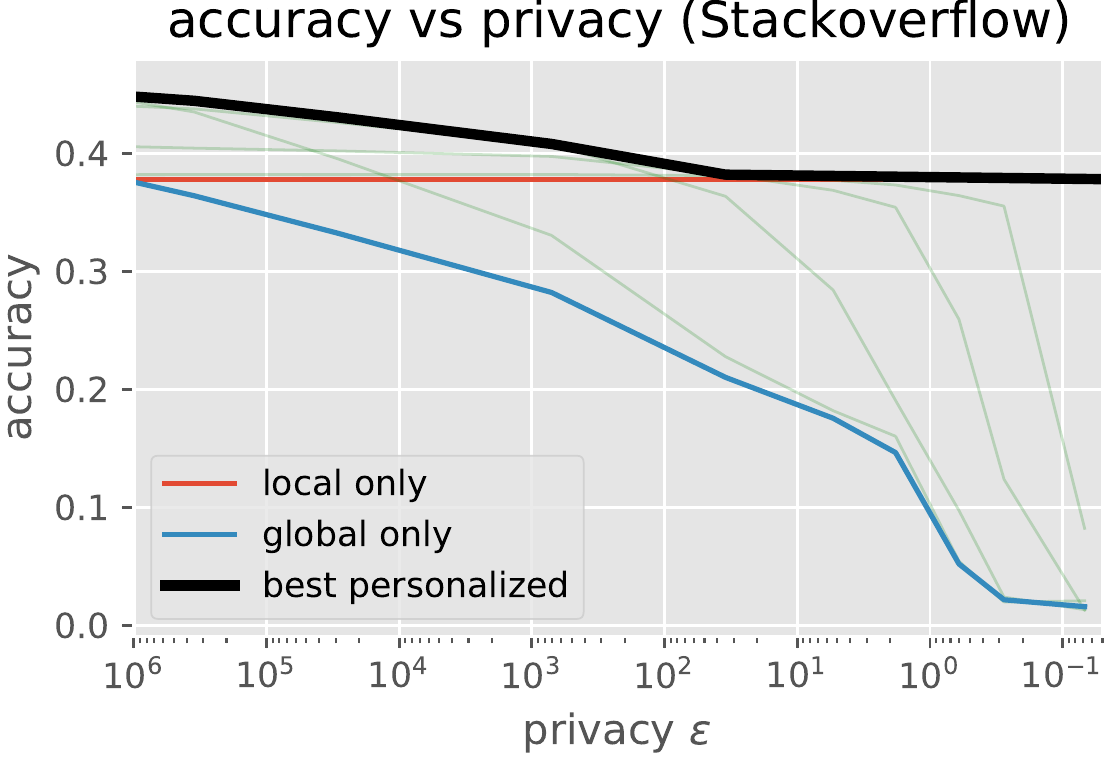}~~~
    \includegraphics[width=.32\textwidth]{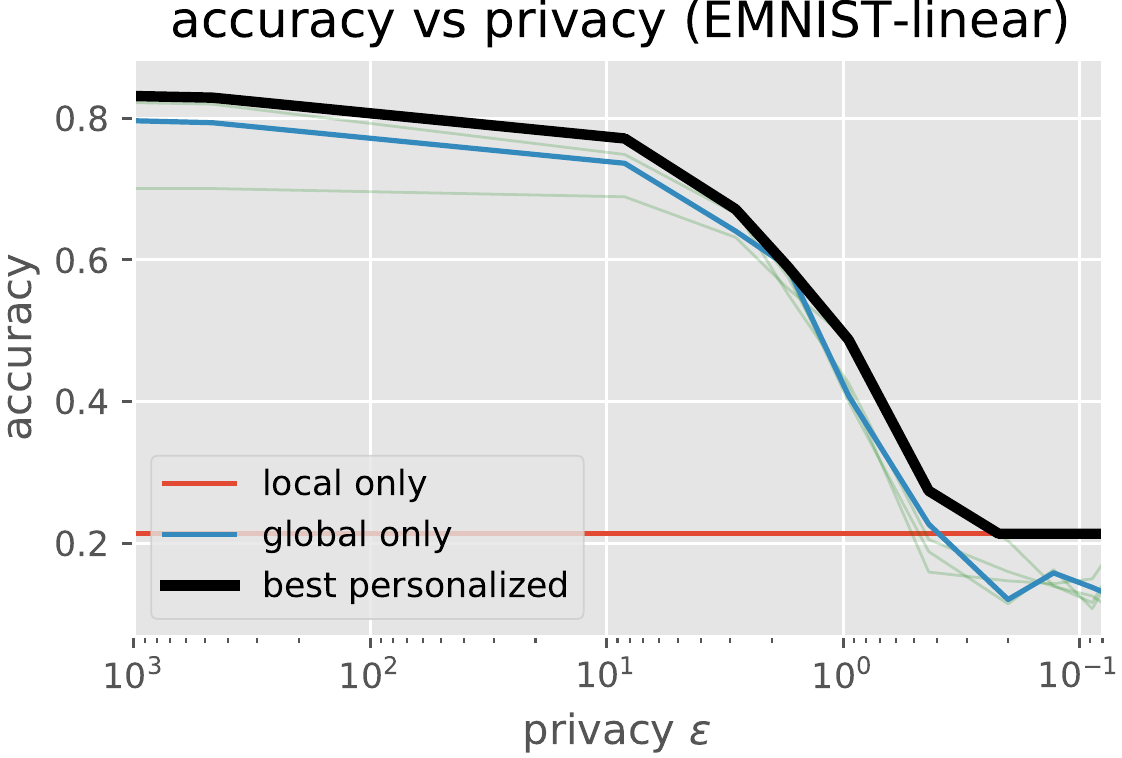}
    \includegraphics[width=.32\textwidth]{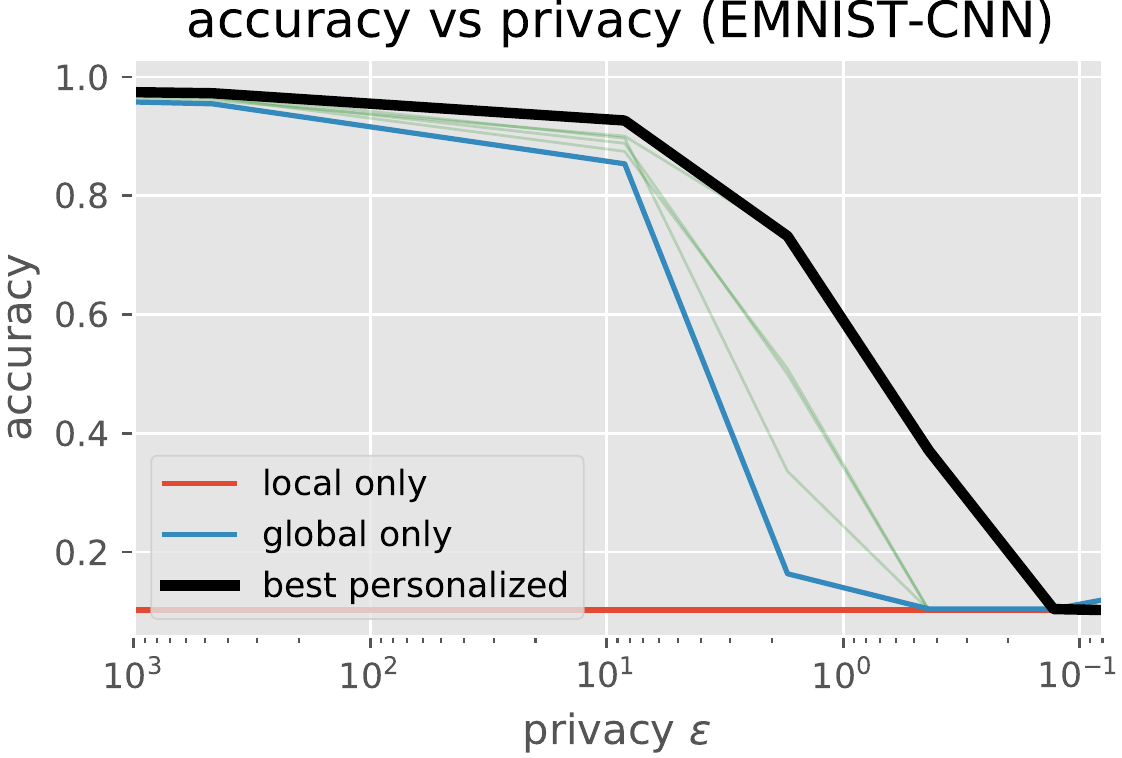}
    \caption{Accuracy--privacy tradeoffs on Stackoverflow and EMNIST at the end of training. The red, blue, and light green curves are obtained by varying the amount of privacy noise~$\sigma$ while optimizing the learning rate with local, global, and intermediate models, respectively. The black curve also optimizes over the level of personalization, which improves the tradeoff.}
    \label{fig:tradeoffs}
\end{figure*}

In this section, we present numerical experiments that illustrate our theoretical guarantees on both synthetic and real-world federated learning datasets.
Our code is available at~\url{https://github.com/albietz/ppsgd}.

\paragraph{Experiment setup.}
We consider linear additive models of the form~\eqref{eq:additive_model} with the squared loss. For classification datasets, we use a one-versus-all setup.
Unless otherwise mentioned, we run our algorithms for one pass over the data, simulating the online setting considered by our theory.
We sample a fixed number of users (denoted~$Q$) at each iteration, chosen randomly without replacement, and consider minibatches of size~$m=10$ for each user.
Following standard practice, we compute privacy guarantees~$\epsilon$ at each iteration using the moments accountant~\citep{abadi2016deep}, with a fixed~$\delta=10^{-4}$, after applying Gaussian noise of variance~$\sigma_\zeta^2 = \sigma^2 C^2$, where~$C$ is a clipping parameter and~$\sigma$ a noise multiplier.
For each run, we consider a fixed step-size~$\eta$, personalization parameter~$\alpha$, clipping parameter~$C$, and noise multiplier~$\sigma$, chosen from a grid (see Appendix~\ref{sec:experiments_appx}).
In order to assess the best achievable performance, we optimize the learning rate separately at each number of iterations reported in our figures.

\paragraph{Synthetic dataset.}
We begin with a synthetic regression example, where each user's data is generated from a ground truth linear model, and multiple coordinates of the parameters are shared across users.
We consider~$N=1000$ users and data in~$d=100$ dimensions.
The parameters~$\theta_i^*$ are generated as follows from a random~$\theta_0 \sim \Ncal(0, 10^2 I_d)$:
\begin{align*}
[\theta^*_i]_{1:95} &= [\theta_0]_{1:95} \\
[\theta^*_i]_{96:100} &= [\theta_0]_{96:100} + \delta_i, \quad \delta_i \sim \Ncal(0, 0.01^2 I),
\end{align*}
while the samples~$(x_i, y_i)$ for user~$i$ are generated by the following process:
\begin{align*}
x_{i} &\sim \Ncal(0, \Sigma), \qquad \Sigma = \diag\{\lambda_k\}_{k=1}^d, \lambda_k = 1/k \\
y_{i} &= \theta_i^{*\top}x_i + \varepsilon_i, \quad \varepsilon_i \sim \Ncal(0, \tau^2).
\end{align*}
With this setup, the excess risk can be computed in closed form, and is given by~$\frac{1}{N} \sum_i \|w + \theta_i - \theta_i^*\|_{\Sigma}^2$, and is shown in Figure~\ref{fig:opt_curves}a.
We observe that global learning ($\alpha=\infty$) enjoys faster convergence at initial iterations in all cases, thanks to better data efficiency, but eventually stops improving, even without privacy noise, since minimizing risk requires personalization in this model.
Small choices of~$\alpha$ lead to better performance in later iterations, and with larger levels of privacy noise, full local learning ($\alpha=0$) dominates joint learning quite early in the learning process, after about 200 iterations, \ie, 2000 samples per user.

\paragraph{Stackoverflow tag prediction.}
We now consider a subset of the Stackoverflow dataset\footnote{\url{https://www.tensorflow.org/federated/api_docs/python/tff/simulation/datasets/stackoverflow/load_data}} consisting of $N=500$ users, each with about~300--500 training and 10--100 test documents in a bag-of-words format of dimension~5000. The task is to predict Stackoverflow tags from documents that consist of user questions, and we use the~80 most popular tags as labels.

In Figure~\ref{fig:opt_curves}b, we show the top-1 accuracy on the test set, while running our method for a single pass over the data, sampling $Q=10$ users and $m=10$ documents per user at each iteration.
We observe that without privacy, the best performance is obtained with an intermediate amount of personalization, $\alpha=10$, suggesting that local models are helpful for generalization on this dataset. With privacy noise, however, models with large~$\alpha$ quickly degrade in performance, and small~$\alpha$ performs best for achieving reasonable privacy levels.
Overall, personalization plays a key role for improving the privacy--accuracy tradeoff on this dataset, as shown in Figure~\ref{fig:tradeoffs}.

\paragraph{Federated EMNIST.}
Finally, we consider a subset of the federated EMNIST digit classification dataset\footnote{\url{https://www.tensorflow.org/federated/api_docs/python/tff/simulation/datasets/emnist/load_data}} consisting of~$N=1000$ users with about $100$ training and $10$ test samples per user.
Here, a single pass does not lead to good performance, suggesting this may be a hard learning problem that benefits from multiple passes~\citep{pillaud2018statistical}.
We thus run our method for 20 epochs, sampling 10 users and 10 images per user at each iteration, and show the resulting test accuracy in Figure~\ref{fig:opt_curves}c.
On this dataset, we see that local learning does generally poorly, perhaps because there is a lot of shared structure in images across users that benefits significantly from global learning.
Nevertheless, there is still a benefit to using small~$\alpha$ for achieving good privacy guarantees, and personalization still plays a role in improving the privacy--accuracy tradeoff (albeit less pronounced than on Stackoverflow, see Figure~\ref{fig:tradeoffs}).

In Figure~\ref{fig:opt_curves}d, we show similar plots on federated EMNIST for a 4-layer CNN with a shared global representation~$\Phi_W(x)$ and an additive model at the output layer, leading to a loss of the form
\begin{equation*}
f_i((w, W), \theta_i, (x, y)) = \ell(y, (w + \theta_i)^\top \Phi_W(x)).
\end{equation*}
We see that similar conclusions hold in this non-convex scenario, namely, small~$\alpha$ is helpful for achieving better privacy guarantees, while large~$\alpha$ leads to better accuracy for poor privacy levels. Here the local-only models ($\alpha = 0$) perform very poorly, since the global representation has frozen random weights, highlighting that learning a useful representation in this model requires a privacy cost.

\section{Discussion and Conclusion}
\label{sec:discussion}

In this paper, we studied personalized private federated optimization algorithms, providing generalization guarantees in the convex setting that highlight the role of personalization and privacy.
We show both theoretically and empirically that adjusting the amount of personalization is always beneficial for improving the accuracy--privacy tradeoff, by providing a way to attenuate the cost of privacy through local learning at large sample sizes.
Promising future directions include extending our analysis to different models of personalization, \eg, based on neural networks, and to provide adaptive algorithms which may dynamically adjust the level of personalization in an online fashion.





\bibliography{full,bibli}
\bibliographystyle{icml2022}

\newpage
\appendix
\onecolumn

\section{Experiment details}
\label{sec:experiments_appx}

We provide the hyperparameter grids for each dataset below.
Our experiments always optimize the step-size at any fixed iteration. To obtain the plots in Figure~\ref{fig:tradeoffs}, we also optimize over~$\alpha$ to obtain the ``best personalized'' curve, while varying over a finer grid of noise levels, provided below.

\textbf{Synthetic dataset:}
\begin{itemize}
    \item step-size~$\eta$: [0.01, 0.02, 0.05, 0.1, 0.2, 0.4, 0.7, 1., 1.2, 1.5, 1.8],
    \item personalization level~$\alpha Q$: [0, 0.1, 0.3, 1, 3, 10, 30, 100, $\infty$],
    \item noise multiplier~$\sigma$: [0, 0.1, 0.3, 1, 3, 10, 30, 100, 300, 1000].
\end{itemize}

\textbf{Stackoverflow:}
\begin{itemize}
    \item step-size~$\eta$:  [2, 5, 10, 20, 50, 100, 150, 200, 250],
    \item personalization level~$\alpha Q$:  [0, 0.1, 0.3, 1, 3, 10, 30, 100, $\infty$],
    \item noise multiplier~$\sigma$:   [0, 0.01, 0.02, 0.05, 0.1, 0.2, 0.5, 1, 2, 5, 10, 100].
\end{itemize}

\textbf{EMNIST-linear:}
\begin{itemize}
    \item step-size~$\eta$:   [0.0001, 0.0002, 0.0005, 0.001, 0.002, 0.005, 0.01],
    \item personalization level~$\alpha Q$:  [0, 0.1, 0.3, 1, 3, 10, 30, 100, $\infty$],
    \item noise multiplier~$\sigma$: [0, 0.1, 0.3, 1, 2, 3, 5, 10, 20, 30, 50, 100].
\end{itemize}

\textbf{EMNIST-CNN:}
\begin{itemize}
    \item step-size~$\eta$:   [0.05, 0.1, 0.2, 0.5, 1., 2., 5.],
    \item personalization level~$\alpha Q$:  [0, 0.1, 0.3, 1, 3, 10, 30, 100, $\infty$],
    \item noise multiplier~$\sigma$: [0, 0.1, 0.3, 1., 3., 10., 30., 100.].
\end{itemize}

\section{Proofs of Main Results}
\label{sec:proofs}


For proving the privacy guarantees, we make use of the following billboard lemma.
 \begin{lemma}[Billboard Lemma,~\citealp{HsuHRRW14}]
     \label{lemma:billboard}
     Suppose a message broadcasting mechanism $\mathcal{M}:\mathcal{D}\rightarrow\mathcal{W}$ is $(\epsilon,\delta)$-differentially private. Consider any set of functions: $g_i:\mathcal{D}_i\times\mathcal{W}\rightarrow\mathcal{W}'$. The composition $\{g_i(\Pi_i\mathcal{D},\mathcal{M}(\mathcal{D}))\}$ is $(\epsilon,\delta)$-joint differentially private, where $\Pi_i:\mathcal{D}\rightarrow\mathcal{D}_i$ is the projection of $\mathcal{D}$ onto $\mathcal{D}_i$.
 \end{lemma}

\subsection{Proof of Theorem~\ref{prop:privacy} (Privacy for PP-SGD)}
\label{sub:privacy_proof}

\begin{proof}
To prove the JDP guarantee, by the billboard lemma (Lemma~\ref{lemma:billboard}), it is sufficient to show that the global model updates satisfy standard DP w.r.t. any changes to a single user's data.

With the assumption that gradients are uniformly bounded by~$G$, each update of the global models has $\ell_2$ sensitivity bounded by~$\alpha \eta/N$ w.r.t.~any change to a single user's data.
If we take
\[\sigma_\zeta \geq c \frac{G\sqrt{n \log(1/\delta)}}{N \epsilon},\]
for some absolute constant~$c$,
then by the Gaussian mechanism, each update is $(\frac{\epsilon}{\sqrt{n \log(1/\delta)}}, \delta)$-differentially private, and by the advanced composition theorem~\citep{dwork2010boosting}, the algorithm that outputs global models is~$(\epsilon, \delta)$-differentially private after~$n$ steps.

When considering both local and global models as outputs of the algorithm, combining the above with the billboard lemma yields the desired JDP guarantee.

\end{proof}

\subsection{Proof of Theorem~\ref{thm:convergence} (Generalization for PP-SGD)}
\label{sub:convergence_proof}

\begin{proof}
Note that the algorithm updates be written jointly on~$z_t = (w_t, \theta_t)$ as
\begin{align}
    z_t &= z_{t-1} - \eta H^{-1} g_t, \\
    \text{with }g_t &= \left( \begin{matrix} \frac{1}{N} \sum_{i=1}^N \nabla_w f_i(w_{t-1}, \theta_{i,t-1}, \xi_{i,t}) + \zeta_t \\ \frac{1}{N} \nabla_\theta f_1( w_{t-1}, \theta_{1,t-1}, \xi_{1,t}) \\
    \vdots \\
    \frac{1}{N} \nabla_\theta f_N( w_{t-1}, \theta_{N,t-1}, \xi_{N,t})
    \end{matrix} \right), \label{eq:gt}
\end{align}
where~$H$ is a pre-conditioner of the form
\begin{equation}
\label{eq:Hdef}
H = \begin{pmatrix}
	\frac{1}{\alpha} I_{\dw} & 0 \\
	0 & I_{N \dth}
\end{pmatrix}.
\end{equation}
Note that~$H$ satisties~$\|z\|^2_H = \langle H z, z \rangle = \|z\|^2_\alpha$, as defined in~\eqref{eq:alpha_norm}.

Denote by~$\mathcal F_{t}$ the sigma algebra spanned by random variables up to time~$t$. Note that we have~$\E[g_t | \mathcal F_{t-1}] = \nabla f(z_{t-1})$.
We also define
\begin{align}
    g_t^* = \left( \begin{matrix} \frac{1}{N} \sum_{i=1}^N \nabla_w f_i(w^*, \theta_i^*, \xi_{i,t}) + \zeta_t \\ \frac{1}{N} \nabla_\theta f_1( w^*, \theta_1^*, \xi_{1,t})  \\
    \vdots \\
    \frac{1}{N} \nabla_\theta f_N( w^*, \theta_N^*, \xi_{N,t})
    \end{matrix} \right), \label{eq:gtstar}
\end{align}
which satisfies~$\E[g_t^*] = \nabla f(z^*) = 0$.
We have

\begin{align}
    \E[\|z_t - z^*\|_H^2 | \Fcal_{t-1}] &= \|z_{t-1} - z^*\|_H^2 - 2 \eta \langle H \E[H^{-1}g_t | \Fcal_{t-1}], z_{t-1} - z^* \rangle + \eta^2 \E[ \|H^{-1} g_t\|_H^2 | \Fcal_{t-1}] \nonumber \\
    &= \|z_{t-1} - z^*\|_H^2 - 2 \eta \langle \nabla f(z_{t-1}), z_{t-1} - z^* \rangle + \eta^2 \E[ \|g_t\|_{H^{-1}}^2 | \Fcal_{t-1}] \nonumber \\
    &\leq \|z_{t-1} - z^*\|_H^2 - 2 \eta (f(z_{t-1}) - f(z^*)) + 2\eta^2 \E[ \|g_t - g_t^*\|_{H^{-1}}^2 | \Fcal_{t-1}] + 2 \eta^2 \E[\|g_t^*\|_{H^{-1}}^2] \\
    &\leq \|z_{t-1} - z^*\|_H^2 - 2 \eta (1 - 2\eta L_\alpha) (f(z_{t-1}) - f(z^*)) + 2 \eta^2\sigma_{tot,\alpha}^2, \label{eq:ztrec}
\end{align}
where the first inequality follows by convexity of~$f$ and using $\|a + b\|^2 \leq 2(\|a\|^2 + \|b\|^2)$, while the second uses Lemma~\ref{lem:smoothness} below to bound~$\E[ \|g_t - g_t^*\|_{H^{-1}}^2 | \Fcal_{t-1}]$, as well as the relation
\begin{align*}
\E[\|g_t^*\|_{H^{-1}}^2]
    &= \alpha \E\left[\left\|\frac{1}{N}\sum_{i=1}^N \nabla_w f_i(w^*, \theta_i^*, \xi_t) + \zeta_t\right\|^2 \right] + \sum_{i=1}^N \E\left[\left\| \frac{1}{N} \nabla_\theta f_i(w^*, \theta_i^*, \xi_t) \right\|^2 \right] \\
    &= \frac{\alpha \bar \sigma_w^2 + \bar \sigma_\theta}{N} + \alpha \dw \sigma_\zeta^2 = \sigma_{tot,\alpha}^2.
\end{align*}
Note that this quantity matches the definition in~\eqref{eq:sigma_tot}.

Assuming~$\eta \leq 1/4L_\alpha$, and taking total expectations, \eqref{eq:ztrec} yields
\[ 
\E[\|z_{t-1} - z^*\|_H^2] \leq \E[\|z_{t-1} - z^*\|_H^2] - \eta \E[f(z_{t-1}) - f(z^*)] + 2 \eta^2 \sigma_{tot,\alpha}^2.
\]
Summing this inequality from~$t=0$ to~$t=n$, we obtain
\begin{align}
    \E[f(\bar z_n) - f(z^*)] \leq \frac{ \|z_0 - z^*\|_\alpha^2}{\eta n} + 2\eta \sigma_{tot,\alpha}^2.
\end{align}
Taking~$\eta = \min\{\frac{1}{4L_\alpha}, \frac{\|z^*\|_\alpha}{\sqrt{n} \sigma_{tot,\alpha}}\}$, with~$z_0 = 0$ yields the desired bound.

\end{proof}

\begin{lemma}
\label{lem:smoothness}
Let $g_t$ and $g_t^*$ be defined as in~\eqref{eq:gt} and~\eqref{eq:gtstar}.
We have
\begin{equation}
    \E[\|g_t - g_t^*\|_{H^{-1}}^2] \leq 2 L_\alpha (f(z_{t-1}) - f(z_t^*)),
\end{equation}
with~$L_\alpha := L \max(\alpha, \frac{1}{N})$.
\end{lemma}

\begin{proof}
   Define $\tilde f_i(w, \theta_i) := f_i(w,\theta_i, \xi_{i,t})$ and $\tilde f(z) := \frac{1}{N}\sum_{i=1}^N \tilde f_i(w, \theta_i)$, which satisfies~$\E[\tilde f(z)] = f(z)$ and~$\E[\nabla \tilde f(z)] = \nabla f(z)$. 
   By $L$-smoothness (Assumption~\ref{ass:conv_smooth}), we have for every $i$, 
   \begin{align*}
       &\|\nabla \tilde f_i(w_{t-1}, \theta_{i,t-1})-\nabla \tilde f_i(w^*, \theta_{i}^*)\|^2 \\
       &\leq 2L\left(\tilde f_i(w_{t-1}, \theta_{i,t-1}) - \tilde f_i(w^*, \theta_i^*) + \langle \nabla_w \tilde f_i(w^*, \theta_i^*), w_{t-1}-w^* \rangle + \langle \nabla_\theta \tilde f_i(w^*, \theta_i^*), \theta_{t-1}-\theta^* \rangle\right) 
   \end{align*}
   Taking the average over $i$, we get 
   \begin{align*}
       \frac{1}{N}\sum_{i=1}^N  \|\nabla \tilde f_i(w_{t-1}, \theta_{i,t-1})-\nabla \tilde f_i(w^*, \theta_{i}^*)\|^2 \leq 2L \left(\tilde f(z_{t-1}) - \tilde f(z^*) + \langle \nabla \tilde f(z^*), z_{t-1} - z^*   \rangle\right). 
   \end{align*}
   Thus, 
   \begin{align*}
       &\|g_t - g_t^*\|_{H^{-1}}^2\\ 
       &= \alpha \left\|\frac{1}{N}\sum_{i=1}^N \left(\nabla_w \tilde f_i(w_{t-1}, \theta_{i,t-1})-\nabla_w \tilde f_i(w^*, \theta_{i}^*)\right) \right\|^2 + \sum_{i=1}^N \left\|\frac{1}{N} \left(\nabla_\theta \tilde f_i(w_{t-1}, \theta_{i,t-1})-\nabla_\theta \tilde f_i(w^*, \theta_{i}^*)\right)\right\|^2 \\
       &\leq \frac{\alpha}{N}\sum_{i=1}^N \left\|\nabla_w \tilde f_i(w_{t-1}, \theta_{i,t-1})-\nabla_w \tilde f_i(w^*, \theta_{i}^*)\right\|^2 + \frac{1}{N^2}\sum_{i=1}^N  \left\|\nabla_\theta \tilde f_i(w_{t-1}, \theta_{i,t-1})-\nabla_\theta \tilde f_i(w^*, \theta_{i}^*)\right\|^2 \tag{by Cauchy-Schwarz inequality} \\
       &\leq \max\left(\alpha, \frac{1}{N}\right)\times \frac{1}{N}\sum_{i=1}^N  \|\nabla \tilde f_i(w_{t-1}, \theta_{i,t-1})-\nabla \tilde f_i(w^*, \theta_{i}^*)\|^2 \\
       &\leq 2L_\alpha \left(\tilde f(z_{t-1}) - \tilde f(z^*) + \langle \nabla \tilde f(z^*), z_{t-1} - z^*   \rangle\right).
   \end{align*}
   with $L_\alpha := L\max(\alpha, \frac{1}{N})$. 
   Taking the expectation, we get
   \begin{align*}
       \E\left[ \|g_t - g_t^*\|_{H^{-1}}^2|\Fcal_{t-1}\right] \leq 2L_\alpha \left(f(z_{t-1}) - f(z^*) + \langle \nabla  f(z^*), z_{t-1} - z^*   \rangle\right) = 2L_\alpha \left(f(z_{t-1}) - f(z^*) \right). 
   \end{align*}
\end{proof}

\subsection{Proof of Lemma~\ref{lem: effect alpha} (effect of~$\alpha$ and critical sample size)}

\begin{proof}
   We would like to make $(w+\theta_i)^\top x=v^\top x$ for all $x$, while minimizing $\|(w, \theta_{1:N})\|_\alpha$. This is achieved by solving the following minimization problem: 
   \begin{align*}
       \min_{w, \theta_{1:N}} \quad & \frac{1}{\alpha}\|w\|^2 + \sum_i \|\theta_i\|^2 \\
       s.t. \quad & w + \theta_i=v 
   \end{align*}
   To solve $w$, we minimize $\frac{1}{\alpha}\|w\|^2 + \sum_i \|v-w\|^2 = \frac{1}{\alpha}\|w\|^2 + N\|v-w\|^2$, which gives $w=\frac{\alpha N}{\alpha N+1}v$ and thus $\theta_i=\frac{1}{\alpha N+1}v$ for all $i$. 
   
   Therefore, 
   \begin{align*}
       \|z^*\|_{\alpha} = \sqrt{\frac{1}{\alpha}\left(\frac{\alpha N}{\alpha N+1}\right)^2 \|v\|^2 + N\left(\frac{1}{\alpha N+1}\right)^2\|v\|^2} = \sqrt{\frac{N}{\alpha N+1}}\|v\|. 
   \end{align*}
   Plugging this into \eqref{eq:excess_risk}, we see that the variance term takes the form specified in the lemma statement. 
   
   To see the effect of $\alpha$ to the variance term, we identify the condition for $\alpha$, such that the following function is increasing in $\alpha$: 
   \begin{align*}
       \psi(\alpha) = \frac{a(\alpha+1)}{\alpha N+1} + \frac{b\alpha}{\alpha N+1}
   \end{align*}
   for constants $a,b$. 
   Its derivative can be calculated as follows: 
   \begin{align*}
       \psi'(\alpha) = \frac{a(\alpha N+1)-a(\alpha+1)N + b(\alpha N+1) - b\alpha N}{(\alpha N+1)^2} = \frac{a(1-N)+b}{(\alpha N+1)^2}. 
   \end{align*}
   Thus, $\psi(\alpha)$ is increasing in $\alpha$ if and only if $a(N-1)\leq b$. In our case, $a=\frac{\sigma^2}{n}$ and $b= c \frac{d_w G^2\log(1/\delta)}{N\epsilon}$, and the condition $a(N-1)\leq b$ is equivalent to 
   \begin{align*}
       n\geq \frac{(N^2-N)\sigma^2\epsilon}{c d_wG^2\log(1/\delta)}. 
   \end{align*}
\end{proof}

\subsection{Proof of Theorem~\ref{prop:privacy_sampling} (privacy of PPSGD with client sampling)}
\label{sub:privacy_sampling_proof}

\begin{proof}
The proof follows the same lines as that of Proposition~\ref{prop:privacy}, but additionally leverages a standard argument of privacy amplification by subsampling to accomodate user sampling. Specifically, we use~\citep[Theorem 1]{abadi2016deep}, which leverages Rényi differential privacy to provide an improved dependency on~$\delta$ compared to the advanced composition theorem~\citep{dwork2010boosting}.

In order to apply~\citep[Theorem 1]{abadi2016deep}, we rewrite the global updates as
\[
w_t = w_{t-1} - \frac{\alpha \eta}{qM} \left(\sum_{i:b_{i,t}=1} \tilde g_{w,i}^t + qM \zeta_t \right).
\]
This matches the form of the the updates in~\citep{abadi2016deep}, with total noise standard deviation~$q M \sigma_\zeta$, which should equal~$\sigma C$ in their notations. Then, by~\citet[Theorem 1]{abadi2016deep}, their condition on~$\sigma$ becomes:
\[
\frac{\sigma_\zeta q M}{C} \geq c_2 \frac{q\sqrt{T \log(1/\delta)}}{\epsilon}.
\]
This guarantees that the global update mechanism after~$T$ iterations is~$(\epsilon, \delta)$-DP. This corresponds to the condition in the statement.
Combining this with the billboard lemma (Lemma~\ref{lemma:billboard}) yields the desired JDP guarantee on the full algorithm~output.

\end{proof}

\subsection{Proof of Theorem~\ref{thm:convergence_sampling} (Generalization of PP-SGD with client sampling)}

\begin{proof}
Note that the algorithm updates be written jointly on~$z_t = (w_t, \theta_t)$ as
\begin{align}
    z_t &= z_{t-1} - \eta H^{-1} g_t, \\
    \text{with }g_t &=  \frac{1}{qM} \left( \begin{matrix}  \sum_{i=1}^N b_{i,t}\sum_{k=1}^{m_i} \nabla_w f_i(w_{t-1}, \theta_{i,t-1}, \xi_{i,t}^{(k)}) + qM \zeta_t \\ b_{1,t} \sum_{k=1}^{m_1}\nabla_\theta f_1( w_{t-1}, \theta_{1,t-1}, \xi_{1,t}^{(k)}) \\
    \vdots \\
    b_{N,t} \sum_{k=1}^{m_N} \nabla_\theta f_N( w_{t-1}, \theta_{N,t-1}, \xi_{N,t}^{(k)})
    \end{matrix} \right), \label{eq:gt_q} \\
\end{align}
with~$H$ as in~\eqref{eq:Hdef}. We have~$\E[g_t | \Fcal_{t-1}] = \nabla f^m(z_{t-1})$.
We also define
\begin{align}
    g_t^* =\frac{1}{qM}  \left( \begin{matrix} \sum_{i=1}^N b_{i,t}\sum_{k=1}^{m_i} \nabla_w f_i(w^*, \theta_i^*, \xi_{i,t}^{(k)}) +  qM \zeta_t  \\
    b_{1,t} \sum_{k=1}^{m_1} \nabla_\theta f_1( w^*, \theta_1^*, \xi_{1,t}^{(k)})  \\
    \vdots \\
    b_{N,t} \sum_{k=1}^{m_N} \nabla_\theta f_N( w^*, \theta_N^*, \xi_{N,t}^{(k)})
    \end{matrix} \right), \label{eq:gtstar_q}
\end{align}
which satisfies~$\E[g_t^*] = \nabla f^m(z^*) = 0$.
We have
\begin{align}
    \E[\|z_t - z^*\|_H^2 | \Fcal_{t-1}] &= \|z_{t-1} - z^*\|_H^2 - 2 \eta \langle H \E[H^{-1} g_t | \Fcal_{t-1}], z_{t-1} - z^* \rangle + \eta^2 \E[ \|H^{-1} g_t\|_H^2 | \Fcal_{t-1}] \nonumber \\
    &= \|z_{t-1} - z^*\|_H^2 - 2 \eta \langle \nabla f^m(z_{t-1}), z_{t-1} - z^* \rangle + \eta^2 \E[ \|g_t\|_{H^{-1}}^2 | \Fcal_{t-1}] \nonumber \\
    &\leq \|z_{t-1} - z^*\|_H^2 - 2 \eta (f^m(z_{t-1}) - f^m(z^*)) + 2\eta^2 \E[ \|g_t - g_t^*\|_{H^{-1}}^2 | \Fcal_{t-1}] + 2 \eta^2 \E[\|g_t^*\|_{H^{-1}}^2].\label{eq:ztrec_q}
\end{align}
By Lemma~\ref{lem:smoothness subsampled}, 
\begin{align}
    \E[ \|g_t - g_t^*\|_{H^{-1}}^2 | \Fcal_{t-1}] \leq 2L_{m,\alpha} \left(f^m(z_{t-1}) - f^m(z^*) \right)  
\end{align}
where $L_{m,\alpha} := L\max\left( \alpha + \frac{\alpha m_{\max}}{qM}, \frac{m_{\max}}{qM}\right)$, 
and
\begin{align}
\E[\|g_t^*\|_{H^{-1}}^2]
    &= \alpha \E\left[\left\|\frac{1}{qM}\sum_{i=1}^N b_{i,t} \sum_{k=1}^{m_i} \nabla_w f_i(w^*, \theta_i^*, \xi_{i,t}^{(k)}) + \zeta_t\right\|^2 \right] + \sum_{i=1}^N \E\left[\left\| \frac{b_{i,t}}{qM} \sum_{k=1}^{m_i} \nabla_\theta f_i(w^*, \theta_i^*, \xi_{i,t}^{(k)}) \right\|^2 \right]\nonumber \\
    &=  \frac{\alpha}{q^2 M^2}\sum_{i=1}^N \E\left[\left\| b_{i,t} \sum_{k=1}^{m_i}\nabla_w f_i(w^*, \theta_i^*, \xi_{i,t}^{(k)}) - q m_i \nabla_w f_i(w^*, \theta_i^*) \right\|^2 \right]+ \frac{1}{qM^2} \sum_{i=1}^N m_i \sigma_{\theta,i}^2 + \alpha \dw \sigma_\zeta^2 \nonumber\\
    &=  \frac{\alpha}{q^2 M^2}\sum_{i=1}^N \E\left[\left\| b_{i,t} \sum_{k=1}^{m_i}\nabla_w f_i(w^*, \theta_i^*, \xi_{i,t}^{(k)}) - b_{i,t} m_i \nabla_w f_i(w^*, \theta_i^*) \right\|^2 \right] \\
    &\qquad + \frac{\alpha}{q^2 M^2}\sum_{i=1}^N \E\left[\left\| (b_{i,t} - q )m_i \nabla_w f_i(w^*, \theta_i^*) \right\|^2 \right] + \frac{1}{qM^2} \sum_{i=1}^N m_i \sigma_{\theta,i}^2 + \alpha \dw \sigma_\zeta^2 \nonumber\\
    &\leq \frac{\alpha }{qM^2}\sum_{i=1}^N m_{i} \sigma_{w,i}^2  + \frac{\alpha}{qM^2}\sum_{i=1}^N m_i^2\E[(b_{i,t}-q)^2]\left\|\nabla_w f_i(w^*, \theta_i^*)\right\|^2 + \frac{\bar \sigma_{\theta,q}^2}{qM} + \alpha \dw \sigma_\zeta^2 \nonumber\\
    &\leq \frac{\alpha \bar\sigma_{w,m}^2 + \bar \sigma_{\theta,m}^2}{qM} + \alpha \dw \sigma_\zeta^2 + \frac{\alpha q(1-q)}{qM^2}\sum_{i=1}^N m_i^2\left\|\nabla_w f_i(w^*, \theta_i^*)\right\|^2 \label{eq:sigma_tot_q}\\
    &= \sigma_{m,\alpha}^2.   \nonumber
\end{align}

Assuming~$\eta \leq 1/4L_{m,\alpha}$, and taking total expectations, \eqref{eq:ztrec_q} yields
\[
\E[\|z_{t-1} - z^*\|_H^2] \leq \E[\|z_{t-1} - z^*\|_H^2] - \eta \E[f^m(z_{t-1}) - f^m(z^*)] + 2 \eta^2 \sigma_{m,\alpha}^2.
\]
Summing this inequality from~$t=0$ to~$t=T$, we obtain
\begin{align}
    \E[f^m(\bar z_T) - f^m(z^*)] \leq \frac{ \|z_0 - z^*\|_\alpha^2}{\eta T} + 2\eta \sigma_{m,\alpha}^2.
\end{align}

Taking $\eta = \min\{\frac{1}{4L_{m,\alpha}}, \frac{\|z^*\|_\alpha}{\sqrt{T} \sigma_{m,\alpha}}\}$, with~$z_0 = 0$ yields the desired bound. 
\end{proof}

\begin{lemma}
\label{lem:smoothness subsampled}
Let $g_t$ and $g_t^*$ be defined as in~\eqref{eq:gt_q} and~\eqref{eq:gtstar_q}.
We have
\begin{equation}
    \E[\|g_t - g_t^*\|_{H^{-1}}^2] \leq 2L_{m,\alpha}\left(f^m(z_{t-1}) - f^m(z^*)\right)  
\end{equation}
with $L_{m,\alpha}:=L\max\left( \alpha + \frac{\alpha m_{\max}}{qM}, \frac{m_{\max}}{qM}\right)$. 
\end{lemma}

\begin{proof}
   Define $\tilde f_i^{(k)}(w, \theta_i) := f_i(w,\theta_i, \xi_{i,t}^{(k)})$ and $\tilde f(z) := \frac{1}{qM}\sum_{i=1}^N b_{i,t} \sum_{k=1}^{m_i}\tilde f_i^{(k)}(w, \theta_i)$, which satisfies~$\E[\tilde f(z)] = f^m(z)$ and~$\E[\nabla \tilde f(z)] = \nabla f^m(z)$. 
   By $L$-smoothness (Assumption~\ref{ass:conv_smooth}), we have for every $i, k$, 
   \begin{align*}
       &\|\nabla \tilde f_i^{(k)}(w_{t-1}, \theta_{i,t-1})-\nabla \tilde f_i^{(k)}(w^*, \theta_{i}^*)\|^2 \\
       &\leq 2L\left(\tilde f_i^{(k)}(w_{t-1}, \theta_{i,t-1}) - \tilde f_i^{(k)}(w^*, \theta_i^*) + \langle \nabla_w \tilde f_i^{(k)}(w^*, \theta_i^*), w_{t-1}-w^* \rangle + \langle \nabla_\theta \tilde f_i^{(k)}(w^*, \theta_i^*), \theta_{t-1}-\theta^* \rangle\right) 
   \end{align*}
   Taking the weighted sum over $i,k$ with weights $\frac{b_{i,t}}{qM}$, we get 
   \begin{align*}
       \frac{1}{qM}\sum_{i=1}^N b_{i,t} \sum_{k=1}^{m_i} \|\nabla \tilde f_i^{(k)}(w_{t-1}, \theta_{i,t-1})-\nabla \tilde f_i^{(k)}(w^*, \theta_{i}^*)\|^2 \leq 2L \left(\tilde f(z_{t-1}) - \tilde f(z^*) + \langle \nabla \tilde f(z^*), z_{t-1} - z^*   \rangle\right). 
   \end{align*}
   Taking the expectation on two sides, we further get 
   \begin{align*}
       \frac{1}{M}\sum_{i=1}^N \sum_{k=1}^{m_i} \E\left[\|\nabla \tilde f_i^{(k)}(w_{t-1}, \theta_{i,t-1})-\nabla \tilde f_i^{(k)}(w^*, \theta_{i}^*)\|^2\right] &\leq 2L \left(f^m(z_{t-1}) - f^m(z^*) + \langle \nabla  f^m(z^*), z_{t-1} - z^*   \rangle\right) \\
       &= 2L \left(f^m(z_{t-1}) - f^m(z^*)\right).  
   \end{align*}
   Thus, 
   \begin{align*}
       &\E\left[\|g_t - g_t^*\|_{H^{-1}}^2\right]\\ 
       &= \alpha \E\left[\left\|\sum_{i=1}^N \frac{b_{i,t}}{qM}\sum_{k=1}^{m_i}\left(\nabla_w \tilde f_i^{(k)}(w_{t-1}, \theta_{i,t-1})-\nabla_w \tilde f_i^{(k)}(w^*, \theta_{i}^*)\right) \right\|^2\right] \\
       &\qquad + \sum_{i=1}^N \E\left[\left\|\frac{b_{i,t}}{qM} \sum_{k=1}^{m_i}\left(\nabla_\theta \tilde f_i^{(k)}(w_{t-1}, \theta_{i,t-1})-\nabla_\theta \tilde f_i^{(k)}(w^*, \theta_{i}^*)\right)\right\|^2\right] \\
       &= \alpha\E\left[\left\|\sum_{i=1}^N \frac{1}{M} \sum_{k=1}^{m_i} \left(\nabla_w \tilde f_i^{(k)}(w_{t-1}, \theta_{i,t-1})-\nabla_w \tilde f_i^{(k)}(w^*, \theta_{i}^*)\right) \right\|^2\right] \\
       &\qquad + \alpha\E\left[ \sum_{i=1}^N \frac{(b_{i,t}-q)^2}{q^2 M^2}\left\|\sum_{k=1}^{m_i}\left(\nabla_w \tilde f_i^{(k)}(w_{t-1}, \theta_{i,t-1})-\nabla_w \tilde f_i^{(k)}(w^*, \theta_{i}^*)\right)\right\|^2 \right] \\
       &\qquad + \sum_{i=1}^N \frac{1}{qM^2} \E\left[\left\|\sum_{k=1}^{m_i} \left(\nabla_\theta \tilde f_i^{(k)}(w_{t-1}, \theta_{i,t-1})-\nabla_\theta \tilde f_i^{(k)}(w^*, \theta_{i}^*)\right)\right\|^2\right]  \\
       &\leq \alpha \sum_{i=1}^N \frac{1}{M} \sum_{k=1}^{m_i}\E\left[\left\|\left(\nabla_w \tilde f_i^{(k)}(w_{t-1}, \theta_{i,t-1})-\nabla_w \tilde f_i^{(k)}(w^*, \theta_{i}^*)\right) \right\|^2\right] \\
       &\qquad + \alpha\sum_{i=1}^N \frac{(1-q)m_i}{qM^2}\sum_{k=1}^{m_i} \E\left[ \left\| \nabla_w \tilde f_i^{(k)}(w_{t-1}, \theta_{i,t-1})-\nabla_w \tilde f_i^{(k)}(w^*, \theta_{i}^*)\right\|^2 \right] \\
       &\qquad + \sum_{i=1}^N \frac{m_{i}}{qM^2} \sum_{k=1}^{m_i}\E\left[\left\|\nabla_\theta \tilde f_i^{(k)}(w_{t-1}, \theta_{i,t-1})-\nabla_\theta \tilde f_i^{(k)}(w^*, \theta_{i}^*)\right\|^2\right]  \\
       &\leq \max\left( \alpha+\frac{\alpha m_{\max}}{qM}, \frac{m_{\max}}{qM}\right)\times \frac{1}{M}\sum_{i=1}^N \sum_{k=1}^{m_i} \E\left[\left\|\nabla \tilde f_i^{(k)}(w_{t-1}, \theta_{i,t-1})-\nabla \tilde f_i^{(k)}(w^*, \theta_{i}^*)\right\|^2\right] \\
       &\leq  \max\left( \alpha + \frac{\alpha m_{\max}}{qM}, \frac{m_{\max}}{qM}\right) \times 2L\left(f^m(z_{t-1}) - f^m(z^*)\right) \\
       &= 2L_{m,\alpha}\left(f^m(z_{t-1}) - f^m(z^*)\right). 
   \end{align*}
\end{proof}

\subsection{Optimizing average user risk with heterogeneous sample sizes}
\label{sub:risk_equal_nonu}

In this section, we study a variant of Algorithm~\ref{alg:sgd_ext} which optimizes the average risk~$f$ instead of the weighted risk~$f^m$ in~\eqref{eq:fq}, as described in Section~\ref{sec:extensions}.
The algorithm is described in Algorithm~\ref{alg:sgd_equal_nonu}, and we provide its generalization and privacy guarantees below.

\begin{algorithm}[tb]
   \caption{PPSGD with client sampling, average user performance}
   \label{alg:sgd_equal_nonu}
\begin{algorithmic}[1]
   \STATE {\bfseries Input:} $q$: client sampling probability,\\
   \hphantom{\bfseries Input:} $m_i$: minibatch sizes, $\eta$: step size,\\
   \hphantom{\bfseries Input:} $\alpha$: global/local ratio,
                               $\sigma_\zeta$: privacy noise level,\\
   \hphantom{\bfseries Input:} $C$: clipping parameter.
   \STATE Initialize $w_0 = \theta_0 = 0$.
   \FOR{$t=1$ {\bfseries to} $T$}
  \STATE Sample~$b_{i,t} \sim Ber(q)$
   \FOR{all clients $i$ with~$b_{i,t} = 1$ in parallel}
   \STATE Sample a minibatch $\{\xi_{i,t}^{(k)}\}_{k=1}^{m_i} \sim P_i^{\otimes m_i}$
   \STATE Compute $g_{\theta,i}^t = \frac{1}{m_i} \sum_{k=1}^{m_i} \nabla_{\theta} f_i(w_{t-1}, \theta_{i,t-1}, \xi_{i,t}^{(k)})$\\
   \hphantom{Compute} $g_{w,i}^t = \frac{1}{m_i} \sum_{k=1}^{m_i} \nabla_w f_i(w_{t-1}, \theta_{i,t-1}, \xi_{i,t}^{(k)})$
   \STATE Update
   $\theta_{i,t} = \theta_{i,t-1} - \frac{\eta}{qN} g_{\theta,i}^t$
   \STATE Clip gradient $\tilde g_{w,i}^t = g_{w,i}^t / \smash{\max(1, \frac{\|g_{w,i}^t\|}{C})}$
   \STATE Send $\tilde g_{w,i}^t$ to the server
   \ENDFOR
   \STATE Sample $\zeta_t \sim \Ncal(0, \sigma_\zeta^2 I_{\dw})$
   \STATE Update $w_t = w_{t-1} - \alpha \eta (\frac{1}{qN}\sum_{i:b_{i,t}=1} \tilde g_{w,i}^t + \zeta_t)$
   \ENDFOR
\end{algorithmic}
\end{algorithm}

\begin{theorem}[Generalization and privacy for Algorithm~\ref{alg:sgd_equal_nonu}]
\label{thm:convergence_sampling_equal}
Under Assumptions~\ref{ass:minimizer},~\ref{ass:conv_smooth}, and~\ref{ass:bounded_grad}, let~$z^* = (w^*, \th^*)$ be any minimizer of~$f$,~$L_{\alpha} = L\max\bigl\{\alpha + \frac{\alpha}{N},\, \frac{1}{N}\bigr\}$, and
\begin{equation}
\label{eq:sigma_q_eq}
    \sigma_{m,\alpha}^2 := \frac{\alpha \bar \sigma_{w,m}^2 + \alpha \tilde \sigma_{w,m}^2 + \bar \sigma_{\theta,m}^2}{qN} + \alpha d_w \sigma_\zeta^2 \leq \frac{\alpha \bar \sigma_{w}^2 + \bar \sigma_{\theta}^2}{qN m_{\min}} + \frac{\alpha \tilde \sigma_{w,m}^2}{qN} + \alpha d_w \sigma_\zeta^2,
\end{equation}
where~$\bar \sigma_{w,m}^2 := \frac{1}{N} \sum_i \frac{\sigma_{w,i}^2}{m_i}$ and~$\bar \sigma_{\theta,m}^2 := \frac{1}{N} \sum_i \frac{\sigma_{\theta,i}^2}{m_i}$ and~$\tilde \sigma_{w,m}^2 := \frac{1}{N} \sum_i q(1-q) \|\nabla_w f_i(w^*, \theta_i^*)\|^2$.

With~$\eta = \min\bigl\{ \frac{1}{4 L_{\alpha}}, \frac{\|z^*\|_\alpha}{\sqrt{T} \sigma_{m,\alpha}}\bigr\}$ and~$C=G$, Algorithm~\ref{alg:sgd_equal_nonu} satisfies
\begin{align*}
\!\!\!\!\!\!
    \E[&f(\bar z_T) - f(z^*)] \leq \frac{4 L_{\alpha} \|z^*\|_\alpha^2}{T} + 3 \frac{\sigma_{m,\alpha}\|z^*\|_\alpha}{\sqrt{T}},
\!\!\!\!\!\!
\end{align*}
with~$\bar z_T = \frac{1}{T} \sum_{t=0}^{T-1} z_t$.

For~$c_1, c_2$ as in Theorem~\ref{prop:privacy_sampling}, and for any~$\epsilon < c_1 q^2 T$, if we take
\[
\sigma_\zeta = c_2 \frac{C \sqrt{T \log(1/\delta)}}{N \epsilon},
\]
then Algorithm~\ref{alg:sgd_equal_nonu} is $(\epsilon,\delta)$-JDP, and the generalization guarantee becomes:
\begin{align}
    \E[f(\bar z_T) - f(z^*)] \lesssim
    \frac{L_{\alpha} \|z^*\|_\alpha^2}{T} + \|z^*\|_\alpha \sqrt{ \frac{ \alpha \bar \sigma_{w,m}^2 + \alpha \tilde \sigma_{w,m}^2 + \bar \sigma_{\theta,m}^2}{q N T} \!+\! \frac{ \alpha \dw G^2 \log(\frac{1}{\delta})}{N^2 \epsilon^2}}.
\end{align}

\end{theorem}

\begin{proof}
   The proof can be obtained by replacing~$f_i$ by~$\tilde f_i = \frac{M}{m_i N} f_i$  in the proofs of Theorem~\ref{thm:convergence_sampling} and Theorem~\ref{prop:privacy_sampling}.
\end{proof}

\end{document}